\documentclass{apple}

\usepackage{amsmath,amssymb,mathtools,bm}
\usepackage{multirow,array,tabularx,adjustbox}
\usepackage{float,enumitem}
\usepackage{tikz,pgfplots}
\usetikzlibrary{arrows.meta,positioning,fit,calc,shapes.geometric,backgrounds,patterns,patterns.meta}
\usepgfplotslibrary{groupplots,fillbetween}
\pgfplotsset{compat=1.18}

\usepackage[colorlinks,linkcolor=red,anchorcolor=blue,citecolor=blue,urlcolor=magenta]{hyperref}
\usepackage[capitalize,noabbrev]{cleveref}
\AtBeginDocument{\urlstyle{sf}}
\hypersetup{
  pdftitle={Rethinking Self-Distillation for Multi-Teacher Capability Merging},
  pdfauthor={Roy Xie, Dan Friedman, Feng Nan, Yukun Huang, Zhichao Xu,
    Chengjiu Zhang, Jun Xu, Manaal Faruqui, Vivek Rathod, Bhuwan Dhingra}
}
\crefformat{section}{\S#2#1#3}
\Crefformat{section}{\S#2#1#3}
\crefformat{subsection}{\S#2#1#3}
\Crefformat{subsection}{\S#2#1#3}
\crefformat{subsubsection}{\S#2#1#3}
\Crefformat{subsubsection}{\S#2#1#3}
\crefrangeformat{section}{\S\S#3#1#4--#5#2#6}
\Crefrangeformat{section}{\S\S#3#1#4--#5#2#6}
\crefmultiformat{section}{\S#2#1#3}{ and #2#1#3}{, #2#1#3}{, and #2#1#3}
\Crefmultiformat{section}{\S#2#1#3}{ and #2#1#3}{, #2#1#3}{, and #2#1#3}
\crefrangeformat{subsection}{\S\S#3#1#4--#5#2#6}
\Crefrangeformat{subsection}{\S\S#3#1#4--#5#2#6}
\crefmultiformat{subsection}{\S#2#1#3}{ and #2#1#3}{, #2#1#3}{, and #2#1#3}
\Crefmultiformat{subsection}{\S#2#1#3}{ and #2#1#3}{, #2#1#3}{, and #2#1#3}

\setlist{nosep,leftmargin=*}
\renewcommand{\arraystretch}{1.1}
\newtheorem{proposition}{Proposition}[section]

\definecolor{cSFT}{HTML}{0072B2}
\definecolor{cKD}{HTML}{009E73}
\definecolor{cOffline}{HTML}{E69F00}
\definecolor{cMOPD}{HTML}{D55E00}
\definecolor{cMixRL}{HTML}{CC79A7}
\definecolor{cNeutral}{HTML}{777777}

\definecolor{cInk}{HTML}{333333}
\definecolor{cGrid}{HTML}{D5D5D5}
\definecolor{cFrame}{HTML}{C7C7C7}

\colorlet{cSFTWash}{cSFT!12}          \colorlet{cSFTPale}{cSFT!50}
\colorlet{cSFTKey}{cSFT!62!black}
\colorlet{cKDWash}{cKD!12}            \colorlet{cKDPale}{cKD!50}
\colorlet{cKDKey}{cKD!62!black}
\colorlet{cOfflineWash}{cOffline!12}  \colorlet{cOfflinePale}{cOffline!50}
\colorlet{cOfflineKey}{cOffline!62!black}
\colorlet{cMOPDWash}{cMOPD!12}        \colorlet{cMOPDPale}{cMOPD!50}
\colorlet{cMOPDKey}{cMOPD!62!black}
\colorlet{cMixRLWash}{cMixRL!12}      \colorlet{cMixRLPale}{cMixRL!50}
\colorlet{cMixRLKey}{cMixRL!62!black}
\colorlet{cNeutralWash}{cNeutral!12}  \colorlet{cNeutralPale}{cNeutral!50}
\colorlet{cNeutralKey}{cNeutral!62!black}

\colorlet{cPending}{cNeutral!55}
\colorlet{cPendingSub}{cNeutral!75}

\DeclareRobustCommand{\barkey}[1]{%
  \tikz[baseline=-.35ex]{%
    \filldraw[fill=#1Pale,draw=#1,line width=.4pt]
      (0,-.5ex) rectangle (1.15ex,1.0ex);
    \filldraw[fill=#1,draw=#1Key,line width=.3pt,
      postaction={pattern={Lines[angle=45,distance=1.3pt,line width=.3pt]},
        pattern color=#1Key}]
      (1.4ex,-.5ex) rectangle (2.55ex,1.0ex);}}

\newcolumntype{Y}{>{\raggedright\arraybackslash}X}
\newcolumntype{C}{>{\centering\arraybackslash}X}
\newcolumntype{R}{>{\raggedleft\arraybackslash}X}

\colorlet{cTableHeader}{gray!20}
\newcommand{\papertablestyle}{%
  \normalsize
  \setlength{\parskip}{0pt}%
  \setlength{\abovecaptionskip}{6pt}%
  \setlength{\belowcaptionskip}{6pt}%
  \setlength{\aboverulesep}{2.5pt}%
  \setlength{\belowrulesep}{2.5pt}%
  \setlength{\cmidrulekern}{3pt}%
  \setlength{\heavyrulewidth}{.8pt}%
  \setlength{\lightrulewidth}{.5pt}%
  \setlength{\cmidrulewidth}{.3pt}%
  \renewcommand{\arraystretch}{1.16}%
}
\newcommand{\tableheadrow}{\rowcolor{cTableHeader}}
\newcommand{\tableheading}[1]{{\bfseries\boldmath#1}}

\newcommand{\hdsize}{\fontsize{6}{7}\selectfont}
\newcommand{\idtag}{\textcolor[HTML]{1a60b0}{\raisebox{1.2pt}{\fontsize{5}{5}\selectfont\ensuremath{\bm{\dagger}}}}}
\newcommand{\oodtag}{\textcolor[HTML]{c0392b}{\raisebox{1.2pt}{\fontsize{5}{5}\selectfont\ensuremath{\bm{\ddagger}}}}}
\newcommand{\recoveryID}[1]{\makebox[\linewidth][c]{\hdsize\bfseries#1\kern.3pt\idtag}}
\newcommand{\recoveryOOD}[1]{\makebox[\linewidth][c]{\hdsize\bfseries#1\kern.3pt\oodtag}}
\newcommand{\colhead}[1]{\makebox[\linewidth][c]{\hdsize\bfseries#1}}

 \newcommand{\method}[1]{\textsc{#1}}
\newcommand{\mopd}{\method{MOPD}}

\newcommand{\offlineopd}{\method{Offline-MOPD}}
\newcommand{\softkd}{\method{Soft-KD}}
\newcommand{\sft}{\method{SFT}}
\newcommand{\mia}{\textsc{MIA}}
\newcommand{\agentthree}{Agent$\times$3}
\newcommand{\mainmodel}{\textsc{Qwen3.5-4B}}
\newcommand{\best}[1]{\textbf{#1}}

\newcommand{\na}{--}
 \newcommand{\repronr}[2][1.55em]{%
  \mbox{#2\kern0.12em\makebox[#1][l]{%
    \raisebox{0.15ex}{\scalebox{0.64}{\normalfont\textcolor{black!55}{\ensuremath{\pm}\kern0.12em n/r}}}}}%
}

 \newcommand{\reproref}[1]{\mbox{#1\phantom{\repronr{}}}}

\newcommand{\reprosim}[3][1.55em]{%
  \mbox{#2\kern0.12em\makebox[#1][l]{%
    \raisebox{0.15ex}{\scalebox{0.64}{\itshape\textcolor{black!55}{\ensuremath{\pm}\kern0.12em#3}}}}}%
}
\newcommand{\Recrate}{\operatorname{Rec}}
\newcommand{\Retain}{\operatorname{Ret}}
\newcommand{\E}{\mathbb{E}}
\newcommand{\KL}{D_{\mathrm{KL}}}

\title{Rethinking Self-Distillation for Multi-Teacher Capability Merging}
\author[1,2]{Roy Xie}
\author[1]{Dan Friedman}
\author[1]{Feng Nan}
\author[1]{Yukun Huang}
\author[1]{Zhichao Xu}
\author[1]{Chengjiu Zhang}
\author[1]{Jun Xu}
\author[1]{\\Manaal Faruqui}
\author[1]{Vivek Rathod}
\author[1,2]{Bhuwan Dhingra}
\affiliation[1]{Apple}
\affiliation[2]{Duke University}
\abstract{
Combining capabilities of multiple expert models trained starting from the same base checkpoint has become increasingly common in frontier language-model post-training. Recent trends suggest that multi-teacher on-policy distillation (MOPD) outperforms conventional off-policy methods.
  However, despite the higher inference and environment interaction costs incurred by MOPD, we find that much of its reported accuracy gain is due to certain training design choices and hyperparameter optimization, rather than the algorithm itself.
  We conduct a controlled self-distillation study across two multi-teacher settings, four models, and eleven benchmarks, comparing off-policy methods, namely supervised fine-tuning (SFT) and soft-label distillation, to hybrid teacher-prefix distillation and MOPD.
  We found that all four methods achieve \textit{nearly identical} accuracy. However, MOPD uses $14.8$--$23.1\times$ SFT's training GPU-hours.
  We also revisit four recently published comparisons between on-policy and off-policy distillation and find that the reported on-policy gains shrink substantially when SFT baselines are trained on rejection-sampled teacher trajectories and use independently tuned hyperparameters.
  As a training-free alternative, we also find that simple weight-merging methods can recover expert capabilities with minutes of CPU merging time, although their accuracy degrades as model size decreases and task interference increases.
  Overall, our results question recent gains reported due to MOPD and suggest careful tuning of more efficient off-policy baselines as a viable alternative.
}
\metadata[Correspondence]{\sffamily\url{ruoyu.xie@duke.edu}}
\metadata[Code]{\url{https://github.com/apple-aiml-research/ml-rethink-mopd}$^{\ast}$}

\begin{document}
\maketitle
\applefootnote{\textcolor{textgray}{\sffamily $^\ast$Under legal review and will be released soon.}}
\vspace{-2.5em}

\section{Introduction}
\label{sec:intro}
Modern frontier large language model (LLM) post-training increasingly decomposes capability development into domain-specific training \citep{dsv4,kimik3,nvidia2026nemotron3}. Starting from a shared checkpoint, each expert can be optimized through reinforcement learning (RL) with its own algorithms, prompts, rewards, and training infrastructure, making capability hill-climbing modular and parallelizable. Transferring these capabilities back into a student initialized from the same checkpoint is also known as \emph{self-distillation} \citep{mai2026thinking}.

However, recent releases report the use of different self-distillation recipes.
Multi-teacher on-policy distillation (MOPD) supervises the student on its own generated trajectories \citep{ma2026mopd,openmopd2026} and is increasingly adopted \citep{lim2026motif3,park2026solaropen2,nvidia2026nemotron3}, including in GLM-5 \citep{glm5team2026}, DeepSeek-V4 \citep{dsv4}, and Kimi K3 \citep{kimik3}. In contrast, MAI-Thinking-1 \citep{mai2026thinking} and Step 3.5 Flash \citep{step2026flash} use off-policy supervised fine-tuning (SFT) on teacher-generated trajectories.

The choice of distillation method affects both the compute budget and the infrastructure required for post-training.
MOPD requires repeated student generation and teacher scoring throughout training \citep{lu2025opd,chen2026regen}. For agentic tasks, each rollout may also require live environment interaction, with costs that grow with trajectory length and the number of environments.
Performing rollouts across multiple environments in the
same training run also requires complex training infrastructure.
Off-policy distillation instead reuses fixed expert trajectories, incurring \textit{one-time} collection costs upfront with no training-time environment interaction.
The two approaches differ substantially in both cost and training dynamics. While some work has suggested that OPD outperforms off-policy methods \citep{ma2026mopd,openmopd2026}, there have been limited systematic comparisons under a controlled setup.
We therefore ask:
\textit{What is the most effective way to consolidate expert capabilities into a single general-purpose model?}

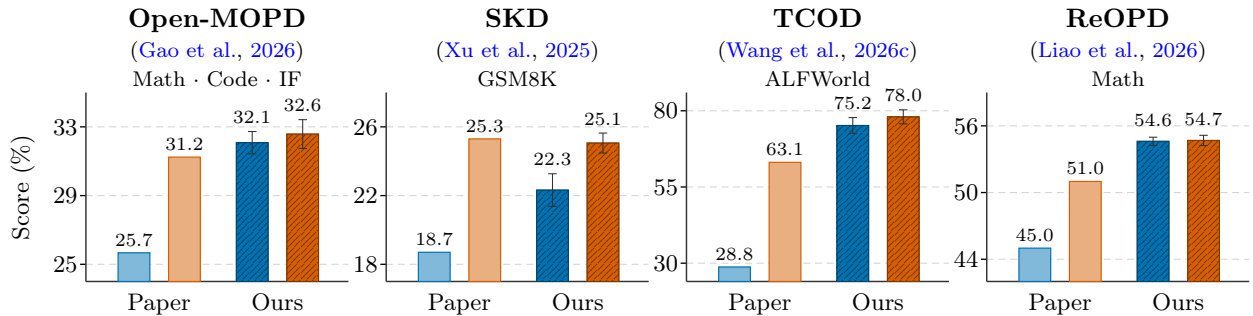
\begin{figure}[t]
\setlength{\parskip}{0pt}
\setlength{\abovecaptionskip}{0pt}
\setlength{\belowcaptionskip}{0pt}
\centering
\pgfmathsetlengthmacro{\reproaxiswidth}{(\linewidth-2.7cm)/4}
\pgfmathsetlengthmacro{\reproaxisstep}{\reproaxiswidth+.5cm}
\noindent\begin{tikzpicture}[font=\normalfont\rmfamily]
\pgfplotsset{reprobaraxis/.style={
  scale only axis,width=\reproaxiswidth,height=2.50cm,
  ybar,bar width=.42cm,
  xmin=-.62,xmax=1.6,xtick={0,1},xticklabels={Paper,Ours},
  axis lines=left,axis line style={-,cInk,line width=.55pt},
  tick align=outside,tick style={black,line width=.5pt},major tick length=2.2pt,
  xtick style={draw=none},
  ymajorgrids=true,
  major grid style={cGrid,line width=.4pt,dash pattern=on 2.6pt off 1.8pt},
  tick label style={font=\normalfont\rmfamily\small,text=black},
  yticklabel style={xshift=2.6pt},
  xticklabel style={font=\normalfont\rmfamily\small,text=black,yshift=1.2pt},
  title style={font=\normalfont\rmfamily\normalsize,align=center,yshift=-1.3mm,
    inner sep=0pt},
  point meta=explicit symbolic,nodes near coords={\pgfplotspointmeta},
  every node near coord/.append style={font=\normalfont\rmfamily\fontsize{7}{8.5}\selectfont,
    text=black,anchor=south,inner sep=0pt,xshift=0pt,yshift=2.5pt},
  clip=false},
  shiftL/.style={bar shift=-9.5pt},
  shiftR/.style={bar shift=9.5pt},
  hatched/.style={postaction={pattern={Lines[angle=45,distance=1.7pt,
    line width=.32pt]},pattern color=#1}},
  papersft/.style={shiftL,fill=cSFTPale,draw=cSFT,line width=.45pt,},
  paperopd/.style={shiftR,fill=cMOPDPale,draw=cMOPD,line width=.45pt,},
  tunedsft/.style={shiftL,fill=cSFT,draw=cSFTKey,line width=.35pt,
    hatched=cSFTKey},
  tunedopd/.style={shiftR,fill=cMOPD,draw=cMOPDKey,line width=.35pt,
    hatched=cMOPDKey},
  ebars/.style={error bars/y dir=both,error bars/y explicit,
    error bars/error bar style={cInk,line width=.4pt},
    error bars/error mark options={rotate=90,cInk,line width=.4pt,
      mark size=1.6pt}},
  labelerror/.style={every node near coord/.append style={
    yshift={max(3pt,#1*\pgfkeysvalueof{/pgfplots/height}/
      (\pgfkeysvalueof{/pgfplots/ymax}-\pgfkeysvalueof{/pgfplots/ymin})+.5pt)}}}}

\begin{axis}[reprobaraxis,name=open,at={(0.00cm,0)},anchor=south west,
  title={\textbf{Open-MOPD}\\[-1pt]{\footnotesize\citep{openmopd2026}}\\[-1pt]{\footnotesize Math $\cdot$ Code $\cdot$ IF}},
  ylabel={Score (\%)},
  ylabel style={font=\normalfont\rmfamily\small,text=black},
  ymin=24,ymax=35,ytick={25,29,33},yticklabels={25,29,33}]
\addplot[papersft] coordinates {(0,25.67) [25.7]};
\addplot[paperopd] coordinates {(0,31.24) [31.2]};
\addplot[tunedsft,ebars,labelerror=.65] coordinates {(1,32.08) +- (0,0.65) [32.1]};
\addplot[tunedopd,ebars,labelerror=.84] coordinates {(1,32.58) +- (0,0.84) [32.6]};
\end{axis}

\begin{axis}[reprobaraxis,name=skd,at={(\reproaxisstep,0)},anchor=south west,
  title={\textbf{SKD}\\[-1pt]{\footnotesize\citep{xu2025skd}}\\[-1pt]{\footnotesize GSM8K}},
  ymin=17,ymax=28,ytick={18,22,26},yticklabels={18,22,26}]
\addplot[papersft] coordinates {(0,18.70) [18.7]};
\addplot[paperopd] coordinates {(0,25.30) [25.3]};
\addplot[tunedsft,ebars,labelerror=.95] coordinates {(1,22.32) +- (0,0.95) [22.3]};
\addplot[tunedopd,ebars,labelerror=.58] coordinates {(1,25.06) +- (0,0.58) [25.1]};
\end{axis}

\begin{axis}[reprobaraxis,name=tcod,at={(2*\reproaxisstep,0)},anchor=south west,
  title={\textbf{TCOD}\\[-1pt]{\footnotesize\citep{wang2026tcod}}\\[-1pt]{\footnotesize ALFWorld}},
  ymin=24,ymax=86,ytick={30,55,80},yticklabels={30,55,80}]
\addplot[papersft] coordinates {(0,28.76) [28.8]};
\addplot[paperopd] coordinates {(0,63.09) [63.1]};
\addplot[tunedsft,ebars,labelerror=2.58] coordinates {(1,75.19) +- (0,2.58) [75.2]};
\addplot[tunedopd,ebars,labelerror=2.31] coordinates {(1,78.02) +- (0,2.31) [78.0]};
\end{axis}

\begin{axis}[reprobaraxis,name=reopd,at={(3*\reproaxisstep,0)},anchor=south west,
  title={\textbf{ReOPD}\\[-1pt]{\footnotesize\citep{liao2026reopd}}\\[-1pt]{\footnotesize Math}},
  ymin=42,ymax=59,ytick={44,50,56},yticklabels={44,50,56}]
\addplot[papersft] coordinates {(0,45.00) [45.0]};
\addplot[paperopd] coordinates {(0,51.00) [51.0]};
\addplot[tunedsft,ebars,labelerror=.37] coordinates {(1,54.61) +- (0,0.37) [54.6]};
\addplot[tunedopd,ebars,labelerror=.46] coordinates {(1,54.69) +- (0,0.46) [54.7]};
\end{axis}
\end{tikzpicture}\par
\vspace{0.75em}
\caption{\textbf{Well-tuned SFT narrows reported OPD advantages.}
Four published SFT--OPD comparisons, with \barkey{cSFT}~SFT and \barkey{cMOPD}~OPD shown as published (pale) and tuned versions (solid hatched): both methods undergo a learning-rate sweep. SFT uses rejection-sampled teacher trajectories. The reported OPD gaps shrink in every comparison; tuning details and full results are in \cref{app:baseline-ablations}. }
\label{fig:repro}
\vspace{-1.5em}

\end{figure}

To assess whether on-policy training is worth its additional cost, we first ask whether its reported gains persist against well-tuned off-policy baselines. Several recent comparisons do not report their SFT tuning procedure \citep{openmopd2026,liao2026reopd,wang2026tcod}. Inspecting the released SKD code \citep{xu2025skd}, we find that it includes prompt tokens in the SFT loss and uses a shared learning rate for SFT and OPD\footnote{\url{https://github.com/google-research/google-research/blob/master/speculative_kd/train/train_sft.py}}.
We revisit four published studies to assess how their reported on-policy gains change with stronger SFT recipes. We tune SFT and OPD learning rates independently and train with rejection-sampled expert trajectories (\cref{sec:repro}). Figure~\ref{fig:repro} shows that these changes substantially reduce the accuracy gaps between SFT and OPD, raising the question of how much of the reported advantage comes from on-policy supervision itself.

Therefore, we conduct a comparison under shared initialization through a controlled self-distillation setup for four methods (\Cref{sec:trainmethods}): two off-policy approaches, SFT on teacher-generated trajectories and soft knowledge distillation (Soft-KD) using full teacher next-token distributions; hybrid Offline-MOPD, which samples student responses conditioned on teacher-generated prefixes; and on-policy MOPD. Shared initialization also makes weight merging a natural alternative: it directly combines expert parameters without training \citep{wortsman2022soups,ilharco2023taskarithmetic}\footnote{Constructing merged weights takes minutes of CPU time, but selecting hyperparameters still requires GPU-hours for validation, as does checkpoint selection for all distillation methods.}. We study two multi-teacher settings: \mia{}, which combines mathematics, instruction following, and agentic tasks; and \agentthree{}, which combines ALFWorld~\citep{shridhar2020alfworld}, WebShop~\citep{yao2022webshop}, and InterCode \citep{yang2023intercode}, where every on-policy rollout requires live environment interaction.

We evaluate four models from the Qwen and Ministral families (1.7B to 8B) on eleven single-turn and multi-turn benchmarks.
We find that, with careful hyperparameter tuning, SFT, Soft-KD, Offline-MOPD, and MOPD achieve comparable accuracy and show similar out-of-domain retention (\Cref{sec:main}). However, MOPD requires $14.8\times$ and $23.1\times$ the training GPU-hours of SFT on \mia{} and \agentthree{}, respectively (\Cref{sec:cost}). As an efficient alternative, weight merging achieves capability recovery comparable to distillation, although its effectiveness declines at smaller model scales and when merging experts with overlapping capabilities (\Cref{sec:merge}). With larger teachers, MOPD recovers a smaller fraction of the teachers' in-domain gains, while SFT maintains high recovery at the cost of greater out-of-domain degradation (\Cref{sec:teacher-scale}). Finally, we analyze attention-weight changes and model behavior to characterize differences among the methods beyond accuracy (\Cref{sec:behavior}). Overall, our results question the reported benefits of MOPD and show that well-tuned off-policy and training-free methods offer efficient alternatives for merging expert capabilities.

\section{Related Work}
\label{sec:related}
\paragraph{Off-Policy and On-Policy Knowledge Distillation.} Off-policy distillation typically trains a student on teacher-generated sequences, using either hard-token targets or teacher probability distributions \citep{hinton2015distilling,kim2016sequence}. Because the student generates its own prefixes during inference, it may encounter mismatch in states from the training data, resulting in accumulating errors \citep{ross2011dagger,bengio2015scheduled}. On-policy distillation (OPD) addresses this mismatch by providing teacher supervision on student-generated trajectories \citep{agarwal2024opd,gu2023minillm,song2026survey}. Intermediate approaches interleave teacher and student tokens \citep{xu2025skd} or combine student sampling with replayed trajectories and prefixes \citep{ko2024distillm,liao2026reopd}. Recent studies also identify teacher–student compatibility \citep{li2026rethinkingopd} and data selection \citep{hou2026whatmatters} as important factors in OPD accuracy, and \citet{ding2026reallydistill} recover most of the OPD gain in an ablation that assigns a fixed negative advantage to selected low-probability student tokens. However, systematic comparisons across these approaches remain limited when multiple teachers and the student share the same initialization. We discuss additional work on SFT and distillation in \cref{app:related}.

\paragraph{Multi-Teacher Capability Integration.}
MOPD combines specialized capabilities by using each domain expert to supervise student-generated trajectories \citep{ma2026mopd,lim2026motif3}. This approach is also used in large-scale post-training \citep{park2026solaropen2,nvidia2026nemotron3}. Alternatives include distilling expert trajectories through SFT \citep{mai2026thinking} and directly combining expert weights without additional training \citep{wortsman2022soups,ilharco2023taskarithmetic,yadav2023ties,yu2024dare}, which \citet{wang2026mixormerge} weigh against joint multi-domain RL. Closest to our work, MOPD \citep{ma2026mopd} and Open-MOPD \citep{openmopd2026} compare on-policy distillation with offline fine-tuning and weight merging for integrating experts trained from a shared initialization. However, these studies do not systematically compare computational costs across methods, and it is unclear to what extent each baseline is independently tuned. For example, using released data and evaluation protocols from Open-MOPD, we find that hyperparameter sweeps and rejection-sampled data can bring SFT to OPD-level accuracy (\cref{sec:repro}).

\section{Methodology}
\label{sec:prelim}
We study \emph{multi-teacher self-distillation} (\cref{fig:pipeline}), which combines the capabilities of $K$ domain experts into a student $\pi_\theta$ initialized from a shared base checkpoint $\pi_0=\pi_{\theta_0}$.
The domain experts $\pi_d=\pi_{\theta_d}$ share this base initialization and provide supervision for student training.
Let $p_d^\theta(\cdot\mid x)$ denote the trajectory distribution induced by policy $\pi_\theta$ on task $x$ in domain $d$, including environment interactions for agentic tasks. A single-turn trajectory $\tau$ contains one response to a fixed prompt. A multi-turn trajectory alternates model actions and environment observations, so later decisions depend on feedback from earlier actions and the growing interaction history.

Off-policy distillation trains the student on recorded teacher trajectories, avoiding student rollout generation during training but restricting supervision to teacher-visited contexts. On-policy distillation instead queries the teacher at student-generated contexts, which can cover states absent from the recorded trajectories but require additional generation and live environment interaction for agentic tasks. To construct the offline corpus $\mathcal D_d$, we sample trajectories from expert $\pi_d$ for domain-$d$ training prompts and retain successful task--trajectory pairs $(x,\tau)$. We collect this rejection-sampled corpus \citep{yuan2023rft} once and hold it fixed across methods: all four use its prompts, SFT and Soft-KD use its recorded trajectories, and Offline-MOPD samples student responses conditioned on its recorded prompts and interaction histories. Teacher trajectory generation costs are discussed in \cref{sec:cost}; we ablate rejection sampling in \cref{sec:training-sensitivity}.

\begin{figure}[t]
\setlength{\parskip}{0pt}
\setlength{\abovecaptionskip}{4pt}
\setlength{\belowcaptionskip}{0pt}
\centering
\noindent\begingroup%
\colorlet{pInk}{cInk}%
\colorlet{pBlue}{cSFT}%
\colorlet{pYellow}{cOffline}%
\colorlet{pRed}{cMOPD}%
\colorlet{pGray}{cNeutral}%
\resizebox{\linewidth}{!}{%
\begin{tikzpicture}[
  x=1cm,y=1cm,font=\normalfont\rmfamily\fontsize{8.5}{10}\selectfont,text=pInk,
  cell/.style={inner sep=0pt},
  body/.style={text=pInk!85},
  head/.style={font=\fontsize{8}{9.5}\selectfont\bfseries,text=black},
  smallnote/.style={font=\fontsize{7.5}{9}\selectfont,text=pInk!75},
  edgelabel/.style={smallnote,inner sep=0pt,outer sep=0pt},
  model/.style={draw=pGray,fill=cNeutralWash,text=black,line width=.5pt,
    rounded corners=2pt,align=center,inner xsep=4pt,inner ysep=2.5pt,
    minimum height=.56cm},
  expert/.style={draw=pYellow,fill=cOfflineWash,text=black,line width=.45pt,
    rounded corners=1.5pt,minimum width=1.15cm,minimum height=.30cm,
    inner xsep=4pt,inner ysep=1pt,
    font=\fontsize{9}{10}\selectfont},
  flow/.style={-{Latex[length=1.3mm,width=.95mm]},
    draw=pInk!65,line width=.5pt,line cap=round,line join=round},
  badge/.style={circle,text=black,
    font=\fontsize{6.2}{6.2}\selectfont\bfseries,
    inner sep=0pt,minimum size=2.6mm,line width=.4pt},
]
\begin{scope}[
  head/.append style={font=\fontsize{9}{10.5}\selectfont\bfseries},
  model/.append style={font=\fontsize{10}{11.5}\selectfont,
    minimum height=.72cm,line width=.6pt},
  expert/.append style={font=\fontsize{11}{12}\selectfont,
    minimum width=1.40cm,minimum height=.40cm},
  smallnote/.append style={font=\fontsize{8.5}{10}\selectfont},
  flow/.append style={line width=.6pt},
  badge/.append style={font=\fontsize{7.2}{7.2}\selectfont\bfseries,
    minimum size=3.0mm},
]
\def\stagegap{1.80cm}
\node[model,minimum width=1.15cm] (base) at (1.30,4.05)
  {\fontsize{14}{15}\selectfont$\pi_0$};
\node[expert,anchor=west] (ifteacher)
  at ([xshift=\stagegap]base.east) {$\pi_{\mathrm{IF}}$};
\node[expert] (math) at ([yshift=.55cm]ifteacher.center) {$\pi_{\mathrm{math}}$};
\node[expert] (agent) at ([yshift=-.70cm]ifteacher.center) {$\pi_{\mathrm{agent}}$};
\node[cell,body,font=\fontsize{10}{10}\selectfont]
  at ([yshift=-.35cm]ifteacher.center) {$\ldots$};
\node[model,draw=pBlue,fill=cSFTWash,minimum width=2.60cm,
  minimum height=1.00cm,anchor=west] (data)
  at ([xshift=\stagegap]ifteacher.east) {};
\draw[cSFTPale,line width=.4pt]
  ([xshift=4pt]data.west) -- ([xshift=-4pt]data.east);
\node[cell,text=black,font=\fontsize{10}{11.5}\selectfont] at ([yshift=.25cm]data.center) {Prompt $x$};
\node[cell,text=black,font=\fontsize{10}{11.5}\selectfont] at ([yshift=-.25cm]data.center) {Trajectory $\tau$};
\node[model,fill=cMOPDWash,draw=pRed,line width=.5pt,
  minimum width=2.30cm,anchor=west] (student)
  at ([xshift=\stagegap]data.east) {Unified student\\[-1pt]$\pi_\theta$};

\coordinate (titlelevel) at (0,5.10);
\node[head,cell] (head1) at (base.center |- titlelevel) {Base model};
\node[head,cell] (head2) at (ifteacher.center |- titlelevel) {Domain experts};
\node[head,cell] (head3) at (data.center |- titlelevel) {Rejection sampling};
\node[head,cell] (head4) at (student.center |- titlelevel) {Capability recovery};
\foreach \i/\badgefill/\badgeborder in {
  1/cNeutralWash/pGray,
  2/cOfflineWash/pYellow,
  3/cSFTWash/pBlue,
  4/cMOPDWash/pRed}{
  \node[badge,fill=\badgefill,draw=\badgeborder,text=black,
    anchor=east] at ([xshift=-3pt]head\i.west) {\i};
}

\foreach \teacher/\offset in {math/.18,ifteacher/0,agent/-.18}{
  \draw[flow] ([yshift=\offset cm]base.east) -- (\teacher.west);
}
\path (base.east) -- coordinate[midway] (rlmid) (ifteacher.west);

\foreach \teacher/\offset in {math/.26,ifteacher/0,agent/-.26}{
  \draw[flow] (\teacher.east) -- ([yshift=\offset cm]data.west);
}
\path (ifteacher.east) -- coordinate[midway] (rolloutmid) (data.west);
\draw[flow] (data.east) -- (student.west);
\path (data.east) -- coordinate[midway] (distillmid) (student.west);

\coordinate (actionlevel) at (0,3.00);
\foreach \midpoint/\action in {rlmid/RL,rolloutmid/Roll out,distillmid/Distill}{
  \node[edgelabel,anchor=base] at (\midpoint |- actionlevel) {\action};
}

\coordinate (initlevel) at (0,2.60);
\draw[flow,dash pattern=on 2.5pt off 2pt,rounded corners=2pt,pInk!50]
  (base.south) -- (base.center |- initlevel) --
  node[smallnote,fill=white,inner xsep=6pt,inner ysep=2pt]
    {Shared initialization: $\theta\leftarrow\theta_0$}
  (student.center |- initlevel) -- (student.south);

\end{scope}

\draw[cFrame,line width=.6pt] (.03,2.18) -- (13.92,2.18);
\def\prefixX{4.70}
\def\rolloutX{8.35}
\def\interactionX{12.00}
\node[head,cell,anchor=west] at (.21,1.91) {Distillation Method};
\node[head,cell] at (\prefixX,1.91) {Prompt / prefix};
\node[head,cell] at (\rolloutX,1.91) {Rollout source};
\node[head,cell] at (\interactionX,1.91) {Multi-turn live interaction};
\draw[cFrame,line width=.45pt] (.03,1.70) -- (13.92,1.70);
\node[cell,body,anchor=west] at (.21,1.44) {\sft{}};
\node[cell,body,anchor=west] at (.21,1.02) {\softkd{}};
\node[cell,body,anchor=west] at (.21,.60) {\offlineopd{}};
\node[cell,body,anchor=west] at (.21,.18) {\mopd{}};

\foreach \y in {1.44,1.02,.60}{
  \node[cell,body] at (\prefixX,\y) {Recorded teacher context};
  \node[cell,body] at (\interactionX,\y) {No};
}
\node[cell,body] at (\prefixX,.18) {Initial task prompt};
\foreach \y in {1.44,1.02}{
  \node[cell,body] at (\rolloutX,\y) {Teacher };
}
\node[cell,body] at (\rolloutX,.60) {Student };
\node[cell,body] at (\rolloutX,.18) {Student };
\node[cell,body] at (\interactionX,.18) {Yes};
\end{tikzpicture}%
}%
\endgroup%
\par
\vspace{1em}

\caption{\textbf{Self-distillation pipeline.}
Each teacher is trained independently from the shared base checkpoint.
All methods use the same rejection-sampled prompts.
Recorded teacher context includes the prompt and prior interaction history.
\emph{Multi-turn live interaction} means executing student actions and receiving fresh
environment observations across turns during training.}
\label{fig:pipeline}

\end{figure}

\subsection{Self-Distillation Objectives}
\label{sec:trainmethods}
Let $s$ denote a token context: the task prompt, preceding interaction history, and current response prefix. The student and teacher next-token distributions are $\pi_\theta(\cdot\mid s)$ and $\pi_d(\cdot\mid s)$ over a shared vocabulary. All per-token losses below apply only to assistant tokens.

\paragraph{Supervised fine-tuning (SFT).}
SFT minimizes the negative log-likelihood of the teacher token $a$ recorded at context $s$:
\begin{equation}
\ell_{\mathrm{SFT}}(\theta;s,a) = -\log\pi_\theta(a\mid s).
\label{eq:sft}
\end{equation}

\paragraph{Soft knowledge distillation (Soft-KD).}
At the same teacher contexts, Soft-KD replaces the single recorded token with the teacher's full next-token distribution and minimizes the forward KL divergence \citep{hinton2015distilling}:
\begin{equation}
\ell_{\mathrm{FKL}}(\theta;d,s)
= \KL\bigl(\pi_d(\cdot\mid s)\bigm\|\pi_\theta(\cdot\mid s)\bigr).
\label{eq:soft-kd}
\end{equation}

\paragraph{Offline-MOPD.}
Offline-MOPD fixes the prompt and preceding interaction history to a recorded teacher trajectory, then samples the current response from the student. Thus $s$ contains a student-generated response prefix, but student responses are used only for distillation: subsequent turns use recorded teacher histories, without executing student actions or obtaining new environment observations for multi-turn scenarios. At these contexts, it minimizes the mode-seeking reverse KL divergence \citep{gu2023minillm}:
\begin{equation}
\ell_{\mathrm{RKL}}(\theta;d,s)
= \KL\bigl(\pi_\theta(\cdot\mid s)\bigm\|\pi_d(\cdot\mid s)\bigr).
\label{eq:reverse-kd}
\end{equation}

\paragraph{Multi-teacher on-policy distillation (MOPD).}
MOPD applies the same reverse KL (\cref{eq:reverse-kd}) to complete student trajectories \citep{agarwal2024opd,ma2026mopd}. Let $\widehat P_d$ be the empirical prompt distribution in $\mathcal D_d$, $w_d\geq0$ the domain weights with $\sum_{d=1}^{K}w_d=1$, and $\bar\theta$ the student parameters at rollout collection. The objective is
\begin{equation}
\mathcal L_{\mathrm{MOPD}}(\theta;\bar\theta)=\sum_{d=1}^{K}w_d\,
\E_{x\sim\widehat P_d,\;\tau\sim p_d^{\bar\theta}(\cdot\mid x)}
\biggl[\frac{1}{|\tau|}\sum_{t=1}^{|\tau|}\ell_{\mathrm{RKL}}(\theta;d,s_t)\biggr],
\label{eq:mopd}
\end{equation}
where $|\tau|$ counts assistant tokens only and $s_t$ precedes the $t$th assistant token. We differentiate only the loss, holding $\bar\theta$ and sampled contexts fixed during each update; subsequent rollouts use the updated student. In multi-turn tasks, the initial instruction is fixed, and subsequent context consists of environment observations produced by executing student actions. All methods share initial task prompts, but MOPD trains on interaction histories generated by student actions. Offline-MOPD instead uses recorded teacher interaction histories, requiring no environment interaction during training. Offline-MOPD shares MOPD's sampling rule and objective on single-turn tasks.

\subsection{Training-Free Weight Merging}
\label{sec:mergemethods}
Weight merging combines expert checkpoints in parameter space \citep{li2025deep}. We consider methods that require no gradient-based training \citep{yadav2023ties,yu2024dare}. For task arithmetic, define the task vectors $\Delta\theta_d=\theta_d-\theta_0$ \citep{ilharco2023taskarithmetic} and combine them as
\begin{equation}
\theta_{\mathrm{merge}} = \theta_0 + \sum_{d=1}^{K} c_d\,\Delta\theta_d,
\label{eq:model-merging}
\end{equation}
where the merging method and its hyperparameters determine the coefficients $c_d\in\mathbb R$ and any sparsification of $\Delta\theta_d$ applied beforehand. We describe the individual merging methods and their sweeps in \cref{app:merging}.

\subsection{Evaluating Capability Recovery}
\label{sec:recovery-eval}
We split tasks into in-domain (ID) and out-of-domain (OOD) based on teachers' training domains. Let $S_b(\pi)$ denote policy $\pi$'s held-out score on benchmark $b$, on a 0--100 scale where higher is better. Define its gain over the base as $\Delta_b(\pi)=S_b(\pi)-S_b(\pi_0)$ and the best teacher gain as $\Delta_b^\star=\max_{1\leq d\leq K}\Delta_b(\pi_d)$. For a nonempty benchmark set $\mathcal B$, we adapt the recovery normalization of \citet{openmopd2026}, measuring student gains over the shared base against the gains of a teacher over the same benchmarks, and define retention as mean score change:
\begin{equation}
\Recrate(\pi;\mathcal B)=100\,\frac{\sum_{b\in\mathcal B}\Delta_b(\pi)}{\sum_{b\in\mathcal B}\Delta_b^\star},
\qquad
\Retain(\pi;\mathcal B)=\frac{1}{|\mathcal B|}\sum_{b\in\mathcal B}\Delta_b(\pi).
\label{eq:recovery-rate}
\end{equation}
On ID benchmarks, $\Recrate$ measures the percentage of aggregate teacher gain recovered: 0 matches the base and 100 matches the aggregate best teacher scores. On OOD benchmarks, $\Retain$ measures mean change from the base in benchmark points, with negative values indicating degradation.

\section{Experimental Setup}
\label{sec:setup}

\subsection{Multi-Teacher Setup}
\label{sec:data}
We study two multi-teacher setups: \mia{} combines mathematics, instruction-following, and agentic experts, mixing disjoint capabilities, while \agentthree{} combines ALFWorld, WebShop, and InterCode-SQL experts \citep{shridhar2020alfworld,yao2022webshop,yang2023intercode} that all focus on long-horizon agentic capability. We study four models under both multi-teacher setups: \textsc{Qwen3-8B}, \textsc{Qwen3-1.7B} \citep{qwen2025qwen3}, \textsc{Qwen3.5-4B-Base} \citep{qwen2026qwen35}, and \textsc{Ministral-3-3B-Instruct} \citep{mistral2025ministral3}; we conduct RL training on each base model to obtain the teachers. The mathematics teacher is trained with DAPO \citep{yu2025dapo} on DAPO-Math-17k \citep{dapomath17k}, the instruction-following teacher with GRPO \citep{shao2024deepseekmath} on the IF-RLVR multi-constraint data \citep{allenai2025ifmulti}, and the three agentic teachers with GiGPO \citep{feng2025gigpo} on the training data of their environments. As a multi-task reference, we include Mix-RL, a single model optimized jointly on all domains \citep{wang2026mixormerge}. All four distillation methods use prompts from the same rejection-sampled teacher corpus; \cref{app:corpus} details its construction.

\subsection{Hyperparameter Tuning}
\label{sec:training}
Our goal is to train each method to its best possible accuracy, so that the comparison is controlled. Each of the four distillation methods is tuned independently and trained for five epochs, as detailed in \cref{app:training-config}. Besides method-specific tuning, such as the reverse-KL coefficient shared by the two MOPD variants, we focus on tuning the learning rate, as the methods differ in gradient scale and a shared rate would handicap some of them. We sweep up to five values per method, model, and setup, giving at most 160 configurations\footnote{$4$ methods $\times$ $4$ models $\times$ $2$ multi-teacher setups $\times$ $5$ learning rates $=$ $160$ configurations.} before seed replication or additional sweeps. Its effect is large: the learning rate alone varies the SFT validation score by up to 11.3 percentage points (\cref{sec:training-sensitivity}). The selected learning rates vary by method, model, and setup (\cref{tab:hparams}). For example, the two off-policy methods are about five times higher than the two on-policy ones, which is why a shared rate would handicap some of them Every training run uses eight B200 GPUs. Weight merging requires no training hyperparameter tuning; we instead sweep each merge method's mixing coefficient and sparsity level and keep the configuration with the highest worst-task retention on validation data. Computing the merged weights uses CPU time, while validation requires GPU inference. Method-specific tuning details are described in \cref{app:method-choices}.

\subsection{Evaluation}
\label{sec:eval}
In our self-distillation setup, every task is split into disjoint training, validation, and test sets. We select the checkpoints with the highest validation-set task accuracy and evaluate them on held-out test sets; \cref{app:selection} details the selection procedure. We report accuracy on eight single-turn benchmarks: MATH-500 \citep{hendrycks2021math,lightman2023verify}, GSM8K \citep{cobbe2021gsm8k}, AIME24 and AIME25 \citep{maa2025aime}, IFEval \citep{zhou2023ifeval}, IFBench \citep{pyatkin2025ifbench}, GPQA-Diamond \citep{rein2024gpqa}, and MMLU-Pro \citep{wang2024mmlupro}; and on three multi-turn benchmarks: InterCode-SQL, ALFWorld, and WebShop, with ALFWorld's seen and unseen splits. \Cref{app:splits} provides the split definitions and the benchmark suite with its metrics and test-set sizes. Every benchmark counts as either in-domain or out-of-domain, based on the teachers' domains. We summarize each model by its recovery rate over all in-domain benchmarks and its retention over all out-of-domain benchmarks, as defined in \cref{sec:recovery-eval}. We report three training seeds for the four distillation methods and Mix-RL, and three evaluation seeds for the base, teachers, and merged models.

\section{Experimental Results}
\label{sec:results}

\subsection{Capability Recovery across Models and Settings}
\label{sec:main}
\label{sec:students}
\label{sec:ood}
\begin{table}[!t]
\papertablestyle
\caption{\textbf{Capability recovery and retention with \textsc{Qwen3-8B}.}
Benchmarks are in-domain (\textcolor[HTML]{1a60b0}{$\bm{\dagger}$}) or out-of-domain (\textcolor[HTML]{c0392b}{$\bm{\ddagger}$}) for each multi-teacher setup.
ID Rec.\ is the recovery rate over all in-domain benchmarks in percent (\cref{eq:recovery-rate}); OOD~$\Delta$ is the mean change from the base model over all OOD benchmarks.
Teachers is a column-wise maximum over that composition's three experts.
Entries report mean $\pm$ standard deviation across three seeds. Bold marks the best of the four distillation methods per column; Merge is the best merge result from \cref{app:merging}. }
\label{tab:capability-recovery}
\label{tab:recovery-qwen3-8b}
\footnotesize
\centering
\begingroup
\setlength{\parskip}{0pt}
\fontsize{6.2}{7.4}\selectfont
\setlength{\tabcolsep}{0.8pt}
\begin{tabularx}{\linewidth}{l*{12}{>{\hsize=0.99\hsize\linewidth=\hsize}R}>{\hsize=1.12\hsize\linewidth=\hsize}R R}
\toprule
 & \multicolumn{8}{c}{\textbf{Single-turn}} & \multicolumn{4}{c}{\textbf{Multi-turn}} & \multicolumn{2}{c}{\textbf{Overall}}\\
\cmidrule(lr){2-9}\cmidrule(lr){10-13}\cmidrule(l){14-15}
\tableheadrow\textbf{\mia{}} & \recoveryID{MATH} & \recoveryID{GSM} & \recoveryID{AIME24} & \recoveryID{AIME25} & \recoveryID{IFEval} & \recoveryID{IFBen} & \recoveryOOD{GPQA} & \recoveryOOD{MMLU} & \recoveryID{IC-SQL} & \recoveryOOD{ALF-s} & \recoveryOOD{ALF-u} & \recoveryOOD{WShop} & \colhead{ID Rec.} & \colhead{OOD $\Delta$}\\
\midrule
Base & \reprosim{82.8}{0.2} & \reprosim{92.4}{0.1} & \reprosim{27.1}{0.7} & \reprosim{21.7}{0.8} & \reprosim{81.9}{0.1} & \reprosim{25.7}{0.3} & \reprosim{42.4}{0.5} & \reprosim{65.6}{0.0} & \reprosim{61.7}{0.1} & \reprosim{18.6}{0.7} & \reprosim{10.4}{0.7} & \reprosim{3.6}{0.2} & \reproref{\na} & \reproref{0.0}\\
Teachers & \reprosim{89.6}{0.7} & \reprosim{93.0}{0.1} & \reprosim{41.2}{1.5} & \reprosim{31.7}{0.8} & \reprosim{88.5}{0.2} & \reprosim{42.0}{0.3} & \reprosim{56.1}{0.5} & \reprosim{69.8}{0.1} & \reprosim{77.6}{0.1} & \reprosim{33.6}{1.2} & \reprosim{32.1}{0.7} & \reprosim{11.8}{0.2} & \reproref{100.0} & \reproref{\na}\\
\midrule[.3pt]
Mix-RL & \reprosim{87.9}{1.2} & \reprosim{91.3}{0.2} & \reprosim{45.3}{1.7} & \reprosim{33.9}{0.7} & \reprosim{88.0}{0.7} & \reprosim{42.4}{1.3} & \reprosim{51.5}{1.7} & \reprosim{60.3}{0.7} & \reprosim{76.4}{0.9} & \reprosim{8.6}{1.2} & \reprosim{5.5}{2.2} & \reprosim{5.8}{0.5} & \reprosim{102.3}{5.2} & \reprosim{$-$0.4}{0.6}\\
Merge & \reprosim{89.2}{0.2} & \reprosim{93.6}{0.0} & \reprosim{40.4}{1.7} & \reprosim{32.9}{1.4} & \reprosim{87.8}{0.2} & \reprosim{41.0}{0.3} & \reprosim{47.1}{0.6} & \reprosim{63.3}{0.1} & \reprosim{77.9}{0.1} & \reprosim{7.9}{0.7} & \reprosim{1.5}{0.7} & \reprosim{6.5}{0.2} & \reprosim{98.9}{2.4} & \reprosim{$-$1.1}{0.3}\\
\midrule[.3pt]
SFT & \reprosim{88.9}{0.2} & \reprosim{92.8}{0.1} & \reprosim{41.7}{1.2} & \reprosim{\best{32.2}}{2.1} & \reprosim{88.5}{0.7} & \reprosim{\best{42.1}}{1.3} & \reprosim{45.6}{0.8} & \reprosim{63.5}{0.2} & \reprosim{76.1}{0.2} & \reprosim{5.5}{1.1} & \reprosim{\best{4.5}}{0.7} & \reproref{0.0} & \reprosim{\best{98.2}}{5.6} & \reprosim{$-$3.0}{0.4}\\
Soft-KD & \reprosim{90.8}{0.9} & \reprosim{92.5}{0.5} & \reprosim{\best{42.5}}{1.2} & \reprosim{30.8}{1.7} & \reprosim{87.4}{0.6} & \reprosim{41.1}{0.4} & \reprosim{47.0}{2.6} & \reprosim{62.7}{0.4} & \reprosim{76.6}{0.3} & \reprosim{5.0}{1.2} & \reprosim{1.0}{1.7} & \reprosim{7.1}{0.3} & \reprosim{97.3}{2.4} & \reprosim{$-$1.6}{0.6}\\
Off-MOPD & \reprosim{\best{92.6}}{0.7} & \reprosim{92.7}{0.4} & \reprosim{41.4}{1.6} & \reprosim{30.3}{2.7} & \reprosim{87.2}{0.8} & \reprosim{40.2}{0.4} & \reprosim{\best{48.8}}{0.8} & \reprosim{63.1}{0.1} & \reprosim{\best{77.0}}{0.7} & \reprosim{\best{7.1}}{1.4} & \reprosim{2.2}{0.7} & \reprosim{\best{7.3}}{1.1} & \reprosim{96.9}{1.7} & \reprosim{\best{$-$0.6}}{0.5}\\
MOPD & \reprosim{89.5}{0.4} & \reprosim{\best{92.9}}{0.3} & \reprosim{41.0}{2.0} & \reprosim{31.5}{0.2} & \reprosim{\best{88.9}}{0.2} & \reprosim{41.7}{1.5} & \reprosim{48.0}{1.3} & \reprosim{\best{65.2}}{0.3} & \reprosim{76.6}{1.1} & \reprosim{6.4}{1.9} & \reprosim{1.7}{0.9} & \reprosim{5.7}{0.6} & \reprosim{97.9}{1.8} & \reprosim{$-$0.8}{0.6}\\
\addlinespace[3pt]
\midrule
\tableheadrow\textbf{\boldmath\agentthree{}} & \recoveryOOD{MATH} & \recoveryOOD{GSM} & \recoveryOOD{AIME24} & \recoveryOOD{AIME25} & \recoveryOOD{IFEval} & \recoveryOOD{IFBen} & \recoveryOOD{GPQA} & \recoveryOOD{MMLU} & \recoveryID{IC-SQL} & \recoveryID{ALF-s} & \recoveryID{ALF-u} & \recoveryID{WShop} & \colhead{ID Rec.} & \colhead{OOD $\Delta$}\\
\midrule
Base & \reprosim{82.8}{0.2} & \reprosim{92.4}{0.1} & \reprosim{27.1}{0.7} & \reprosim{21.7}{0.8} & \reprosim{81.9}{0.1} & \reprosim{25.7}{0.3} & \reprosim{42.4}{0.5} & \reprosim{65.6}{0.0} & \reprosim{61.7}{0.1} & \reprosim{18.6}{0.7} & \reprosim{10.4}{0.7} & \reprosim{3.6}{0.2} & \reproref{\na} & \reproref{0.0}\\
Teachers & \reprosim{84.8}{0.7} & \reprosim{92.7}{0.0} & \reprosim{29.4}{0.9} & \reprosim{22.3}{1.0} & \reprosim{84.5}{0.2} & \reprosim{27.7}{0.3} & \reprosim{49.0}{0.5} & \reprosim{65.3}{0.1} & \reprosim{78.6}{0.1} & \reprosim{82.1}{1.4} & \reprosim{80.6}{0.7} & \reprosim{62.2}{0.2} & \reproref{100.0} & \reproref{\na}\\
\midrule[.3pt]
Mix-RL & \reprosim{85.3}{0.4} & \reprosim{95.2}{0.2} & \reprosim{31.0}{0.9} & \reprosim{23.7}{1.1} & \reprosim{87.4}{0.8} & \reprosim{25.7}{0.9} & \reprosim{44.6}{1.3} & \reprosim{63.6}{0.1} & \reprosim{77.7}{1.1} & \reprosim{85.0}{1.4} & \reprosim{82.8}{1.3} & \reprosim{61.0}{0.5} & \reprosim{100.3}{1.2} & \reprosim{$+$2.1}{0.1}\\
Merge & \reprosim{84.5}{0.2} & \reprosim{93.9}{0.1} & \reprosim{25.9}{0.5} & \reprosim{20.5}{0.8} & \reprosim{84.8}{0.2} & \reprosim{28.0}{0.3} & \reprosim{43.6}{0.6} & \reprosim{64.0}{0.0} & \reprosim{78.5}{0.1} & \reprosim{81.4}{1.2} & \reprosim{82.3}{1.1} & \reprosim{60.4}{0.5} & \reprosim{99.0}{0.3} & \reprosim{$+$0.7}{0.2}\\
\midrule[.3pt]
SFT & \reprosim{80.2}{0.4} & \reprosim{91.9}{0.2} & \reprosim{24.9}{2.0} & \reprosim{20.3}{0.8} & \reprosim{82.0}{0.1} & \reprosim{20.8}{1.0} & \reprosim{41.1}{1.1} & \reprosim{62.5}{0.2} & \reprosim{77.5}{0.5} & \reprosim{80.7}{2.5} & \reprosim{\best{79.9}}{2.2} & \reprosim{\best{63.0}}{0.7} & \reprosim{99.0}{1.7} & \reprosim{$-$2.0}{0.2}\\
Soft-KD & \reprosim{79.2}{1.4} & \reprosim{92.4}{0.4} & \reprosim{19.0}{2.0} & \reprosim{18.6}{2.8} & \reprosim{78.6}{0.6} & \reprosim{20.9}{0.5} & \reprosim{43.4}{1.3} & \reprosim{61.0}{0.5} & \reprosim{78.3}{0.2} & \reprosim{79.3}{2.1} & \reprosim{79.1}{1.3} & \reprosim{61.2}{1.4} & \reprosim{97.6}{2.1} & \reprosim{$-$3.3}{0.8}\\
Off-MOPD & \reprosim{82.4}{2.3} & \reprosim{92.7}{0.2} & \reprosim{25.6}{1.0} & \reprosim{22.1}{1.9} & \reprosim{82.7}{0.6} & \reprosim{25.1}{0.8} & \reprosim{\best{45.6}}{1.1} & \reprosim{\best{64.9}}{0.1} & \reprosim{78.2}{1.3} & \reprosim{80.0}{0.7} & \reprosim{77.6}{1.5} & \reprosim{62.0}{1.4} & \reprosim{97.8}{1.7} & \reprosim{$+$0.2}{0.9}\\
MOPD & \reprosim{\best{83.7}}{1.3} & \reprosim{\best{92.9}}{0.2} & \reprosim{\best{26.9}}{0.9} & \reprosim{\best{22.4}}{1.4} & \reprosim{\best{83.9}}{0.5} & \reprosim{\best{26.1}}{0.4} & \reprosim{44.9}{1.0} & \reprosim{63.6}{0.3} & \reprosim{\best{79.3}}{1.3} & \reprosim{\best{83.6}}{1.9} & \reprosim{78.4}{3.0} & \reprosim{60.8}{0.5} & \reprosim{\best{99.3}}{1.0} & \reprosim{\best{$+$0.6}}{0.4}\\
\bottomrule
\end{tabularx}
\endgroup
\end{table}

We present capability recovery and retention for \textsc{Qwen3-8B} in \cref{tab:capability-recovery} to assess whether on-policy distillation recovers more expert capability than off-policy methods.
\mia{} mixes single-turn mathematics and instruction following with multi-turn InterCode-SQL, whereas \agentthree{} contains only multi-turn tasks.
SFT, Soft-KD, Offline-MOPD, and MOPD recover 96.9--98.2\% of the teachers' aggregate in-domain gains on \mia{} and 97.6--99.3\% on \agentthree{}.
Across all four models (\cref{tab:aggregate-models}), the maximum within-model spread in in-domain recovery rate is 3.5 percentage points on \mia{} and 4.6 on \agentthree{}. These comparisons show no consistent in-domain recovery advantage for on-policy distillation over tuned off-policy methods, and a larger teacher does not create one (\cref{sec:teacher-scale}). We discuss this result and give a theoretical analysis in \Cref{app:offpolicy-theory}.
OOD retention varies by model and setting: MOPD improves mean retention over SFT by 1.1 benchmark points, but does not consistently outperform it. Mix-RL and merging have mean OOD changes of $-1.1$ and $-2.7$ points, respectively.
Appendix~\ref{app:per-model} reports the same per-benchmark results for the other three models.

\begin{table}[t]
\papertablestyle
  \caption{\textbf{Recovery and retention across models.} Model-specific columns report in-domain recovery (\%); OOD~$\Delta$ averages the reported retention equally over all eight model--setting pairs. Full benchmark results are in Appendix~\ref{app:per-model}. }
  \label{tab:aggregate-models}
\footnotesize
  \centering
  \footnotesize
  \setlength{\tabcolsep}{2pt}
  \begin{tabularx}{\linewidth}{l*{9}{R}}
  \toprule
   & \multicolumn{2}{c}{\tableheading{\textsc{Qwen3-1.7B}}} & \multicolumn{2}{c}{\tableheading{\textsc{Ministral-3-3B}}} & \multicolumn{2}{c}{\tableheading{\textsc{Qwen3.5-4B}}} & \multicolumn{2}{c}{\tableheading{\textsc{Qwen3-8B}}} & \tableheading{Mean}\\
  \cmidrule(lr){2-3}\cmidrule(lr){4-5}\cmidrule(lr){6-7}\cmidrule(lr){8-9}
  \tableheadrow\tableheading{Method} & \multicolumn{1}{c}{\tableheading{\mia{}}} & \multicolumn{1}{c}{\tableheading{A$\times$3}} & \multicolumn{1}{c}{\tableheading{\mia{}}} & \multicolumn{1}{c}{\tableheading{A$\times$3}} & \multicolumn{1}{c}{\tableheading{\mia{}}} & \multicolumn{1}{c}{\tableheading{A$\times$3}} & \multicolumn{1}{c}{\tableheading{\mia{}}} & \multicolumn{1}{c}{\tableheading{A$\times$3}} & \multicolumn{1}{c}{\tableheading{OOD $\Delta$}}\\
  \midrule
  Mix-RL & \reprosim{92.3}{1.3} & \reprosim{101.9}{0.8} & \reprosim{93.2}{4.8} & \reprosim{92.7}{1.3} & \reprosim{103.5}{4.1} & \reprosim{102.6}{1.2} & \reprosim{102.3}{5.2} & \reprosim{100.3}{1.2} & $-1.1$\\
  Merge & \reprosim{86.2}{0.2} & \reprosim{88.3}{0.3} & \reprosim{99.7}{1.7} & \reprosim{77.6}{0.4} & \reprosim{101.8}{2.5} & \reprosim{86.8}{0.5} & \reprosim{98.9}{2.4} & \reprosim{99.0}{0.3} & $-2.7$\\
  \midrule[.3pt]
  SFT & \reprosim{94.7}{2.9} & \reprosim{\best{98.3}}{0.7} & \reprosim{95.5}{1.1} & \reprosim{95.1}{0.3} & \reprosim{98.6}{3.0} & \reprosim{96.4}{1.1} & \reprosim{\best{98.2}}{5.6} & \reprosim{99.0}{1.7} & $-3.6$\\
  Soft-KD & \reprosim{\best{98.2}}{2.0} & \reprosim{96.8}{1.7} & \reprosim{\best{98.4}}{1.3} & \reprosim{93.3}{1.2} & \reprosim{97.3}{3.4} & \reprosim{\best{97.0}}{1.0} & \reprosim{97.3}{2.4} & \reprosim{97.6}{2.1} & $-4.0$\\
  Off-MOPD & \reprosim{97.7}{2.3} & \reprosim{96.3}{2.1} & \reprosim{94.9}{4.4} & \reprosim{95.4}{1.3} & \reprosim{\best{98.7}}{3.2} & \reprosim{94.8}{1.3} & \reprosim{96.9}{1.7} & \reprosim{97.8}{1.7} & $-3.3$\\
  MOPD & \reprosim{96.9}{1.8} & \reprosim{96.8}{0.6} & \reprosim{97.0}{2.6} & \reprosim{\best{97.9}}{1.9} & \reprosim{97.8}{4.8} & \reprosim{95.9}{1.6} & \reprosim{97.9}{1.8} & \reprosim{\best{99.3}}{1.0} & \best{$-2.5$}\\
  \bottomrule
  \end{tabularx}
  \end{table}

\subsection{Distillation Efficiency}
\label{sec:cost}
\begin{figure}[t]
\setlength{\parskip}{0pt}
\setlength{\abovecaptionskip}{4pt}
\setlength{\belowcaptionskip}{0pt}
\centering
\includegraphics[width=\linewidth]{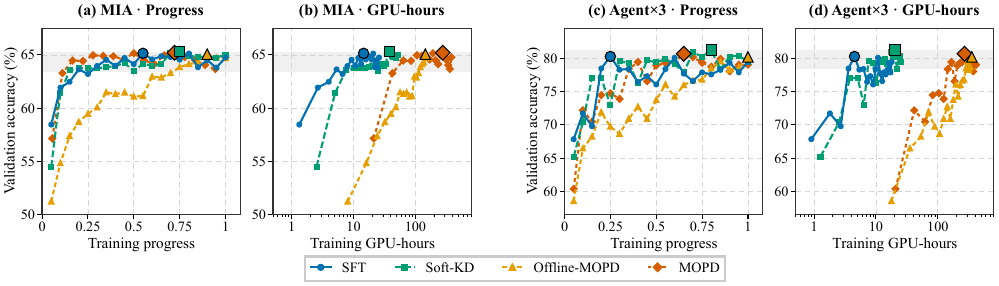}\par
\caption{\textbf{Validation accuracy versus training progress and GPU-hours.}
Accuracy is averaged across domains.
Training progress is the update count divided by the full-run budget; GPU-hours are cumulative and shown on a logarithmic scale. Markers identify the highest retained validation score for each method. }
\label{fig:cost}
\end{figure}

\citet{ma2026mopd} report each method's accuracy against training steps.
However, step counts alone do not capture computational cost across distillation methods when each one is individually tuned. For example, learning rates affect convergence speed, batch sizes determine steps per epoch, and on-policy updates also require student generation and teacher scoring.
We therefore report cumulative GPU-hours on eight B200 GPUs (\cref{app:training-config}).
\Cref{fig:cost} compares validation accuracy against training progress and cumulative GPU-hours for \textsc{Qwen3-8B}. At the full configured budgets, MOPD and SFT use 383.32 and 25.90 GPU-hours on \mia{} ($14.8\times$), and 415.80 and 18.00 GPU-hours on \agentthree{} ($23.1\times$), respectively (\cref{tab:compute}). Student rollout accounts for 86.5--96.0\% of OPD step time (\cref{tab:opd-step-timing}). The outlined markers retain the original best validation checkpoints; their costs exclude subsequent training and checkpoint selection.
It is worth noting that reported training GPU-hours do not include one-time teacher trajectory generation, as the teacher costs for on-policy and off-policy methods are not directly comparable.
SFT's data-generation cost is incurred once and amortized across multiple training runs and hyperparameter searches.
On the other hand, MOPD requires teacher forward passes to score new student trajectories in every training run. We discuss these two operational differences and report the generation cost in more detail in \cref{app:teacher-generation}.

\subsection{Alternatives to Distillation}
\label{sec:merge}

We compare distillation methods with weight merging and Mix-RL to assess whether expert capabilities can be recovered \textit{without} distillation training.
Weight merging is more effective at the larger model scale: averaged over both settings, \textsc{Qwen3-8B} recovers 99.0\% of the teachers' gains, compared with 87.3\% for \textsc{Qwen3-1.7B} (\cref{tab:aggregate-models}).
Merging also performs better in \mia{}, where the experts specialize in less overlapping domains: mean recovery across the four models is 96.6\%, compared with 87.9\% on \agentthree{}.
Constructing merged weights takes minutes of CPU, but selecting a merge still requires GPU validation (\cref{app:merging}).
Joint multi-domain RL (Mix-RL) also avoids a separate distillation stage by starting a single monolithic RL training run.
It recovers most of the teachers' gains across all eight model--setting pairs, approaching or exceeding teacher-level performance on \textsc{Qwen3.5-4B} and \textsc{Qwen3-8B} (\cref{tab:aggregate-models}).
However, changes to one domain's training recipe may require retuning the joint run, whereas separate experts support independent iteration before integration \citep{ma2026mopd}.

\section{Analysis and Discussion}
\label{sec:analysis}

\subsection{Do Larger Teachers Favor On-Policy Distillation?}
\label{sec:teacher-scale}
\begin{table}[!t]
\papertablestyle
\caption{\textbf{Teacher size with a fixed \textsc{Qwen3-1.7B} student.}
Each method distills from either 1.7B or 8B Qwen3 experts.
ID Rec.\ is the percentage of the corresponding teachers' aggregate in-domain gains over the 1.7B base recovered by the student (\cref{eq:recovery-rate}); }
\label{tab:scale}
\footnotesize
\centering
\footnotesize
\setlength{\tabcolsep}{2pt}
\begin{tabularx}{\linewidth}{l*{8}{R}}
\toprule
 & \multicolumn{2}{c}{\tableheading{\mia{} ID Rec.}} & \multicolumn{2}{c}{\tableheading{\mia{} OOD $\Delta$}} & \multicolumn{2}{c}{\tableheading{\agentthree{} ID Rec.}} & \multicolumn{2}{c}{\tableheading{\agentthree{} OOD $\Delta$}}\\
\cmidrule(lr){2-3}\cmidrule(lr){4-5}\cmidrule(lr){6-7}\cmidrule(lr){8-9}
\tableheadrow\tableheading{Method} & \multicolumn{1}{c}{\tableheading{1.7B}} & \multicolumn{1}{c}{\tableheading{8B}} & \multicolumn{1}{c}{\tableheading{1.7B}} & \multicolumn{1}{c}{\tableheading{8B}} & \multicolumn{1}{c}{\tableheading{1.7B}} & \multicolumn{1}{c}{\tableheading{8B}} & \multicolumn{1}{c}{\tableheading{1.7B}} & \multicolumn{1}{c}{\tableheading{8B}}\\
\midrule
SFT & \reprosim{94.7}{2.9} & \reprosim{92.5}{1.6} & \reprosim{$-$6.3}{0.3} & \reprosim{$-$9.2}{0.5} & \reprosim{98.3}{0.7} & \reprosim{99.3}{0.5} & \reprosim{$-$5.4}{0.9} & \reprosim{$-$14.8}{0.7}\\
Soft-KD & \reprosim{98.2}{2.0} & \reprosim{88.6}{1.5} & \reprosim{$-$7.7}{0.2} & \reprosim{$-$8.4}{0.1} & \reprosim{96.8}{1.7} & \reprosim{93.0}{1.2} & \reprosim{$-$4.1}{0.5} & \reprosim{$-$11.7}{0.7}\\
Off-MOPD & \reprosim{97.7}{2.3} & \reprosim{80.9}{1.0} & \reprosim{$-$6.9}{0.5} & \reprosim{$-$3.0}{0.8} & \reprosim{96.3}{2.1} & \reprosim{78.9}{0.9} & \reprosim{$-$2.6}{0.2} & \reprosim{$-$1.1}{0.7}\\
MOPD & \reprosim{96.9}{1.8} & \reprosim{77.4}{1.9} & \reprosim{$-$6.1}{0.8} & \reprosim{$-$2.3}{0.3} & \reprosim{96.8}{0.6} & \reprosim{73.3}{0.3} & \reprosim{$-$2.1}{0.5} & \reprosim{$-$3.3}{0.6}\\
\bottomrule
\end{tabularx}
\vspace{0.25em}
\end{table}

A larger teacher may provide stronger supervision while differing more from the student's initial behavior.
To examine whether using larger teachers favors on-policy distillation, we compare 1.7B and 8B experts with a fixed \textsc{Qwen3-1.7B} student (\cref{tab:scale}).
With 1.7B teachers, the four methods achieve similar in-domain recovery: 94.7--98.2\% on \mia{} and 96.3--98.3\% on \agentthree{}\footnote{Recovery is normalized by each teacher composition's gains over the student base, so comparisons across teacher sizes use different denominators.}. With 8B teachers, SFT and Soft-KD have higher in-domain recovery than both MOPD variants on \mia{} and the reported \agentthree{} results.
The gap is larger on \agentthree{}: SFT reaches 99.3\% and Soft-KD 93.0\%, whereas Offline-MOPD and MOPD reach 78.9\% and 73.3\%, respectively. On \agentthree{}, however, this recovery advantage comes at a large out-of-domain cost.
With 8B instead of 1.7B teachers, out-of-domain degradation grows from 5.4 to 14.8 points for SFT and from 4.1 to 11.7 for Soft-KD, with minimal changes on Offline-MOPD (2.6 to 1.1) and MOPD (2.1 to 3.3).
On \mia{}, SFT and Soft-KD experience larger OOD losses at 9.2 and 8.4 points, compared with 3.0 for Offline-MOPD and 2.3 for MOPD.
\Cref{tab:recovery-qwen3-1p7b-8bteacher} provides the per-benchmark results for the 8B teachers.

\subsection{Parameter Updates and Behavioral Changes}
\label{sec:behavior}
\begin{figure}[t]
\setlength{\parskip}{0pt}
\setlength{\abovecaptionskip}{2pt}
\setlength{\belowcaptionskip}{0pt}
\centering
\includegraphics[width=\linewidth]{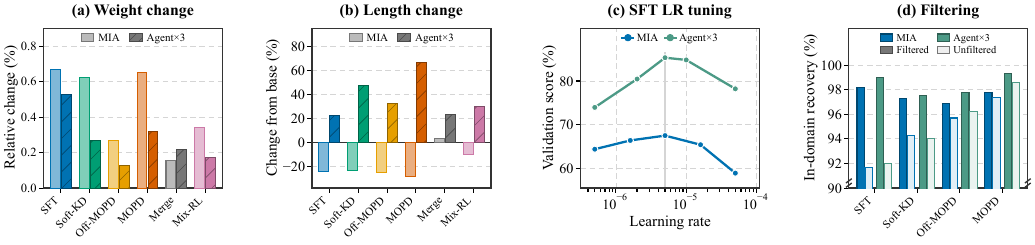}\par
\caption{\textbf{Model changes and training sensitivity on Qwen3-8B.}
\textbf{(a)} Attention-weight changes remain below $0.7\%$ across methods in both settings.
\textbf{(b)} Distillation shortens responses on \mia{} and lengthens them on \agentthree{}, with MOPD showing the largest response-length changes.
\textbf{(c)} SFT validation scores peak at an intermediate learning rate in both settings and vary by up to $11.3$ percentage points across the sweep.
\textbf{(d)} Rejection sampling yields the largest recovery gains for off-policy distillation methods. }
\label{fig:diagnostics}
\end{figure}

To assess how much each method alters the base model, we measure the Frobenius norm of the attention-weight update relative to that of the base weights (\cref{fig:diagnostics}a).
On \textsc{Qwen3-8B}, all relative changes are below $0.7\%$, with merging yielding the smallest changes on average.
Small weight changes coexist with substantial shifts in response length (\cref{fig:diagnostics}b).
Using geometric means of paired response-length ratios on in-domain tasks, distillation shortens responses by $23$--$29\%$ on \mia{} and lengthens them by $22$--$67\%$ on \agentthree{}.
MOPD produces the largest response-length change in both settings.
Merging increases response length by $3\%$ and $23\%$, respectively.
Overall, similar recovery scores can accompany different parameter updates and generation behavior.

\subsection{What Does It Take to Make SFT a Strong Baseline?}
\label{sec:training-sensitivity}
The strength of an SFT baseline depends on optimization and training data quality.
For \textsc{Qwen3-8B}, we sweep five learning rates, selecting the best validation checkpoint within five epochs at each rate (\cref{fig:diagnostics}c). Both settings peak at $5\times10^{-6}$; validation scores span $8.5$ percentage points on \mia{} and $11.3$ on \agentthree{}.
These differences show why SFT must be tuned independently before a performance gap can be attributed to the distillation objective. We also examine rejection sampling, which retains only correct or successful teacher trajectories. SFT recovery gaps are $6.5$ percentage points on \mia{} and $7.0$ on \agentthree{}, compared with $3.0$ and $3.6$ for Soft-KD and $0.5$--$1.6$ for Offline-MOPD and MOPD (\cref{fig:diagnostics}d).
Each method has a slightly larger gain on \agentthree{} than on \mia{}.
The result shows filtering benefits off-policy more than on-policy methods.

\subsection{Do Reported On-Policy Gains Persist with Stronger SFT Baselines?}
\label{sec:repro}
\begin{table}[t]
\papertablestyle
\centering
\caption{\textbf{Reproducing four published SFT--OPD comparisons.}
Baseline SFT is our initial reproduction; tuned SFT uses the original or unfiltered corpus; RS SFT additionally uses rejection-sampled teacher data. The two $\Delta$ columns compare successive SFT recipes. OPD's advantage over RS SFT is smaller than published in every comparison.}
\label{tab:main-ladder}
\fontsize{8}{9.5}\selectfont
\newcolumntype{G}{>{\hspace{4pt}}c}
\newcommand{\tfgroup}[1]{\multicolumn{2}{>{\hspace{4pt}}c}{\tableheading{#1}}}
\setlength{\tabcolsep}{2.5pt}
\resizebox{\linewidth}{!}{%
\begin{tabular}{lGcGcGcGcGc}
\toprule
 & \tfgroup{Teacher} & \tfgroup{Baseline SFT} & \tfgroup{OPD} & \tfgroup{Tuned SFT} & \tfgroup{RS SFT}\\
\cmidrule(l{6.5pt}r{2.5pt}){2-3}\cmidrule(l{6.5pt}r{2.5pt}){4-5}\cmidrule(l{6.5pt}r{2.5pt}){6-7}\cmidrule(l{6.5pt}r{2.5pt}){8-9}\cmidrule(l{6.5pt}r{0pt}){10-11}
\tableheadrow\tableheading{Study} & \tableheading{Paper} & \tableheading{Ours} & \tableheading{Paper} & \tableheading{Ours} & \tableheading{Paper} & \tableheading{Ours} & \tableheading{Score} & \tableheading{$\Delta$} & \tableheading{Score} & \tableheading{$\Delta$}\\
\midrule
Open-MOPD & 32.4 & \reprosim[20pt]{32.5}{0.2} & 25.7 & \reprosim[20pt]{23.8}{0.4} & 31.2 & \reprosim[20pt]{32.6}{0.8} & \reprosim[20pt]{25.1}{0.3} & $+1.3$ & \reprosim[20pt]{32.1}{0.6} & $+7.0$\\
SKD & 60.0 & \reprosim[20pt]{58.6}{1.1} & 18.7 & \reprosim[20pt]{18.4}{1.0} & 25.3 & \reprosim[20pt]{25.1}{0.6} & \reprosim[20pt]{23.1}{0.9} & $+4.7$ & \reprosim[20pt]{22.3}{1.0} & $-0.8$\\
TCOD & 81.3 & \reprosim[20pt]{84.3}{1.5} & 28.8 & \reprosim[20pt]{65.7}{0.9} & 63.1 & \reprosim[20pt]{78.0}{2.3} & \reprosim[20pt]{70.8}{2.6} & $+5.1$ & \reprosim[20pt]{75.2}{2.6} & $+4.4$\\
ReOPD & 49.9 & \reprosim[20pt]{55.6}{0.6} & 45.0 & \reprosim[20pt]{47.1}{1.2} & 51.0 & \reprosim[20pt]{54.7}{0.5} & \reprosim[20pt]{54.0}{0.8} & $+6.9$ & \reprosim[20pt]{54.6}{0.4} & $+0.6$\\
\bottomrule
\end{tabular}}
\end{table}

We revisit the SFT--OPD comparisons in Open-MOPD \citep{openmopd2026}, SKD \citep{xu2025skd}, TCOD \citep{wang2026tcod}, and ReOPD \citep{liao2026reopd} to assess whether reported on-policy gains persist with independently tuned SFT baselines and rejection-sampled teacher trajectories.
Only Open-MOPD shares our multi-teacher self-distillation setting, with domain experts trained from the student's initial checkpoint; the SKD, TCOD, and ReOPD settings we reproduce each distill a single larger teacher into a smaller student.
The results show SFT gains of 1.3--6.9 percentage points after tuning (\cref{tab:main-ladder}). Relative to tuned SFT, the RS recipes gain 7.0 points on Open-MOPD, 4.4 on TCOD, and 0.6 on ReOPD, and lose 0.8 on SKD.
Open-MOPD, TCOD, and ReOPD reproductions use reporting benchmarks for tuning; our main study uses held-out validation sets.
These recipe comparisons also differ in data, training budgets, or evaluation details (\cref{app:baseline-ablations}).

\section{Conclusion}
\label{sec:conclusion}
Across four models and two task compositions, we find no consistent in-domain recovery advantage for MOPD over tuned off-policy distillation. For \textsc{Qwen3-8B}, MOPD uses $14.8\times$ and $23.1\times$ SFT's training GPU-hours at the full configured budgets on \mia{} and \agentthree{}, respectively. Revisiting published comparisons also narrows their reported on-policy gains, with residual gaps that depend on the task and baseline recipe. 
Weight merging offers a training-free alternative, but its performance varies with model size and task composition.
Overall, we question recent gains reported due to MOPD and suggest careful tuning of more efficient off-policy and training-free baselines as viable and cheaper alternatives. We discuss limitations in \cref{app:limitations}.

\bibliographystyle{plainnat}
\bibliography{references}

\clearpage
\appendix
\crefalias{section}{appendix}
\crefalias{subsection}{appendix}
\crefalias{subsubsection}{appendix}
\makeappendixtitle

\newlength{\mergenumw}
\newcommand{\mgSel}{\raisebox{.1ex}{\fontsize{5.5}{5.5}\selectfont\ensuremath{\bigstar}}}
\newcommand{\mgGray}[1]{\textcolor{black!55}{#1}}
\newcommand{\mgNeg}[1]{\makebox[0pt][r]{\ensuremath{-}}#1}
\newcommand{\mgNA}{\makebox[\mergenumw][c]{--}}
\newcommand{\mgPanel}[2]{\tableheadrow & \multicolumn{6}{l}{\textbf{#1}\hspace{8pt}{\normalfont\scriptsize\textcolor{black!60}{#2}}}\\}
\newcommand{\mgHeadroom}[3]{& \multicolumn{2}{l}{\mgGray{\itshape Teacher gain over base (pts)}} & \mgGray{\itshape #1} & \mgGray{\itshape #2} & \mgGray{\itshape #3} & \\}

\setlength{\intextsep}{10pt plus 2pt minus 2pt}
\makeatletter
\setlength{\@fptop}{0pt}
\makeatother
\section{Reproducing Four Published SFT--OPD Comparisons}
\label{app:repro}
\label{app:baseline-ablations}
We describe the reproduction setups for the four SFT--OPD comparisons in
\cref{fig:repro,tab:main-ladder}. The subsections report individual benchmark or split results
with the procedures used to obtain them. We use released artifacts where
available. Our scope is only SFT--OPD comparisons, not an assessment of each study's proposed method or broader
contributions.

\subsection{Open-MOPD}
\label{app:repro-openmopd}

\paragraph{Models and data.}
Open-MOPD integrates mathematics, code, and instruction-following capabilities by
using domain experts to supervise student-generated responses \citep{openmopd2026}.
We initialize both methods from the released SmolLM3-3B MixSFT checkpoint and reuse
the released experts and evaluators. MixSFT is the published SFT reference,
distinct from the paper's rejection-fine-tuning (RFT) baseline. The published MixSFT
recipe uses unfiltered R1 and Nemotron traces and roughly balances response tokens
across mathematics, code, and instruction following (37:28:35, approximately 1:1:1).
The naive OPD sampler instead allocates prompts approximately 2:2:1; these are
different units of mixture balance.

\paragraph{Training and evaluation.}
Following the released recipe, we train the SFT baselines up to seven
epochs. Learning rates and epoch budgets are selected on the reporting benchmarks
and then fixed across three RS SFT training runs. We report their seven-epoch
checkpoints and three OPD runs at update 200, using the released OPD rate of
$1.5\times10^{-6}$. The original and rejection-sampled SFT recipes differ in both
teacher data and training budget. We evaluate on AIME 2024/2025, LiveCodeBench v5/v6, and IFEval/IFBench. \Cref{tab:omopd-ladder} provides the
individual benchmark scores.

\begin{table}[H]
\papertablestyle
\centering
\caption{\textbf{Open-MOPD results by benchmark.}
Domain means average two benchmarks; Avg. weights the three domains equally.
Published rows are taken from Table~2 of \citet{openmopd2026}: RouteRL,
MixSFT at epoch~4, and Open-MOPD. Teacher rows use routed experts.
RS denotes rejection sampling.
}
\label{tab:omopd-ladder}
\footnotesize
\setlength{\tabcolsep}{2.5pt}
\resizebox{\textwidth}{!}{%
\begin{tabular}{l rrr rrr rrr r}
\toprule
 & \multicolumn{3}{c}{\tableheading{Math}} & \multicolumn{3}{c}{\tableheading{Code}} & \multicolumn{3}{c}{\tableheading{IF}} & \\
\cmidrule(lr){2-4}\cmidrule(lr){5-7}\cmidrule(lr){8-10}
\tableheadrow\tableheading{Method} & \tableheading{AIME24} & \tableheading{AIME25} & \tableheading{Mean} & \tableheading{LCBv5} & \tableheading{LCBv6} & \tableheading{Mean} & \tableheading{IFEval} & \tableheading{IFBench} & \tableheading{Mean} & \tableheading{Avg.}\\
\midrule
Teacher, published & 23.7 & 24.8 & 24.2 & 22.2 & 21.3 & 21.7 & 74.5 & 27.7 & 51.1 & 32.4\\
Teacher, reproduced & 24.1 & 27.1 & 25.6 & 23.2 & 22.3 & 22.7 & 74.8 & 23.3 & 49.0 & 32.5\\
\midrule
\sft{} (MixSFT), published & 15.6 & 20.3 & 17.9 & 16.0 & 19.2 & 17.6 & 66.9 & 16.0 & 41.5 & 25.7\\
OPD, published & 22.0 & 22.9 & 22.4 & 20.8 & 22.6 & 21.7 & 74.5 & 24.7 & 49.6 & 31.2\\
\midrule
\sft{} baseline, reproduced & 11.9 & 17.5 & 14.7 & 14.7 & 18.4 & 16.5 & 64.1 & 15.9 & 40.0 & 23.8\\
Tuned \sft{}, original corpus & 13.1 & 19.8 & 16.5 & 15.6 & 20.0 & 17.8 & 65.9 & 16.4 & 41.2 & 25.1\\
Tuned \sft{}, RS & 26.8 & 25.3 & 26.1 & 24.1 & 23.5 & 23.8 & 70.9 & 21.8 & 46.4 & 32.1\\
OPD, reproduced & 24.4 & 26.2 & 25.3 & 24.4 & 22.4 & 23.4 & 75.2 & 22.8 & 49.0 & 32.6\\
\midrule
OPD $-$ RS \sft{} & $-2.4$ & $+0.9$ & $-0.8$ & $+0.3$ & $-1.1$ & $-0.4$ & $+4.3$ & $+1.0$ & $+2.6$ & $+0.5$\\
\bottomrule
\end{tabular}}
\end{table}

\subsection{SKD}
\label{app:repro-skd}

\paragraph{Models and data.}
SKD interleaves student-proposed tokens with teacher replacements to construct
distillation trajectories \citep{xu2025skd}. We reproduce its SFT and OPD
baselines on GSM8K. We use a Gemma-2B-it student and a Gemma-7B-it teacher. The release provides
training code but no teacher checkpoint, so we retrain the teacher on the
authors' 7{,}000-example training set. The published SFT baseline uses 1{,}000 gold
solutions, while OPD supervises student responses to the same prompts. For
teacher-data SFT, we generate one greedy response per prompt and compare the
full corpus with a rejection-sampled corpus that retains responses matching
the gold answer. Filtering drops unsolved prompts without resampling, changing
both corpus size and prompt coverage.

\paragraph{Training and evaluation.}
The released SFT loss includes prompt tokens, whereas the distillation trainer
supervises only responses. We train SFT with response-only loss, disable packing,
and match the distillation batch size of eight. The release uses a shared learning
rate of $10^{-5}$ without a documented SFT-specific sweep; we select learning
rates independently and select checkpoints using validation accuracy.
In separate matched controls, each method receives 375 updates at global batch
size eight. These controls use two teacher draws and three training seeds, with
learning rates selected for each method and teacher using a separate sweep seed
and the grid $\{3\times10^{-6},5\times10^{-6},10^{-5},2\times10^{-5},5\times10^{-5}\}$.
Checkpoint selection
maximizes development accuracy, breaking ties in favor of the earlier step. We evaluate on the 1{,}319-item GSM8K test set with greedy decoding, a 512-token
limit, and a double-newline stop sequence. Baseline reproductions use the released
grader, which requires the entire answer line after \texttt{\#\#\#\#} to match
the gold answer. Teacher-data comparisons and matched controls instead extract
the first number on that line, preventing trailing role markers from invalidating
a correct numeric response.

\subsection{TCOD}
\label{app:repro-tcod}

\paragraph{Models and data.}
TCOD gradually expands the portion of an interaction trajectory assigned to the
student during on-policy distillation \citep{wang2026tcod}. We reproduce its SFT
and OPD baselines on ALFWorld. We use a Qwen2.5-3B-Instruct student. Both
methods use an ALFWorld-trained Qwen2.5-7B-Instruct teacher. We collect teacher trajectories on the authors' 3{,}553 training
games, allowing up to six attempts per game and retaining 3{,}342 successful
trajectories for SFT. OPD generates student rollouts on the full training game
list. The release provides OPD code; we reconstruct the SFT pipeline from the
teacher trajectories.

\paragraph{Training and evaluation.}
Our SFT recipe treats each teacher action as a separate target and supervises
actions throughout the episode. This replaces the framework's default multi-turn
formatting, which supervises only the final assistant message. SFT trains for two
epochs with a 4{,}096-token limit, while OPD trains for 250 updates. We tune the
learning rate for each method on the reporting splits, selecting $2\times10^{-5}$
for both; the released OPD rate is $10^{-6}$. The selected rates and training
budgets are fixed across subsequent runs, and we report terminal checkpoints.
The SFT development split is constructed by turn and can share games with training. Both methods use the released \emph{accumulate} protocol, which retains the full
dialogue and repeats the two latest observation--action pairs in each new prompt.

\begin{table}[H]
  \papertablestyle
  \centering
  \caption{\textbf{TCOD results by ALFWorld split.}
  Scores are success rates (\%). Avg. gives the equal-weight mean of seen and unseen.
}
  \label{tab:tcod-ladder}
  \footnotesize
  \setlength{\tabcolsep}{7pt}
  \begin{tabular}{lrrr}
  \toprule
  \tableheadrow\tableheading{Method} & \tableheading{Seen (140)} & \tableheading{Unseen (134)} & \tableheading{Avg.}\\
  \midrule
  Teacher, published & 85.7 & 76.9 & 81.3\\
  Teacher, reproduced & 84.2 & 84.3 & 84.3\\
  \midrule
  \sft{}, published & 32.1 & 25.4 & 28.8\\
  OPD, published & 65.7 & 60.4 & 63.1\\
  \midrule
  \sft{} baseline, reproduced & 68.6 & 62.9 & 65.7\\
  Tuned \sft{}, without RS & 73.1 & 68.4 & 70.8\\
  Tuned \sft{}, RS & 77.4 & 73.0 & 75.2\\
  OPD, reproduced & 81.4 & 74.6 & 78.0\\
  \midrule
  OPD $-$ RS \sft{} & +4.0 & +1.6 & +2.8\\
  \bottomrule
  \end{tabular}
  \end{table}

\subsection{ReOPD}
\label{app:repro-reopd}

\paragraph{Models and data.}
ReOPD reuses teacher interaction histories as prefixes for student sampling during
distillation, avoiding fresh environment rollouts \citep{liao2026reopd}. We reproduce
its SFT and OPD baselines in the mathematics setting. We use the released
Qwen3-4B-Instruct-2507 student, initialized by supervised
training on 2{,}000 ReTool trajectories, and the RL-trained Qwen3-8B teacher.
Teacher training and student distillation use 6{,}400 DAPO prompts. The release
provides checkpoints and OPD code; we reconstruct the teacher-trajectory SFT
pipeline. We generate four teacher candidates per training prompt, constructing
an unfiltered corpus of 6{,}335 trajectories and a rejection-sampled corpus of
5{,}071.

\paragraph{Training and evaluation.}
SFT supervises every assistant turn while masking prompts and tool observations.
OPD supervises student-generated rollouts through the Python-tool environment.
We train SFT for five epochs, giving 2{,}274 updates on the unfiltered corpus and
1{,}719 on the rejection-sampled corpus. Independent learning-rate sweeps select
$5\times10^{-6}$ for unfiltered SFT, $10^{-5}$ for rejection-sampled SFT, and
$3\times10^{-6}$ for OPD, compared with the published SFT and OPD rates of
$5\times10^{-5}$ and $10^{-6}$. We follow the mathematics evaluation script with temperature 1.0, top-$p$ 1,
and at most 16 tool calls. Scores average AIME 2024/2025, AMC 2023, MATH500,
Minerva Math, and OlympiadBench with equal weights.
\Cref{tab:reopd-ladder} reports the individual benchmark scores.

\begin{table}[H]
\papertablestyle
\centering
\caption{\textbf{ReOPD results by mathematics benchmark.}
Avg. weights the six benchmarks equally. Reproduced benchmark values use the
original evaluator.}
\label{tab:reopd-ladder}
\footnotesize
\setlength{\tabcolsep}{3.5pt}
\resizebox{\textwidth}{!}{%
\begin{tabular}{lrrrrrrr}
\toprule
\tableheadrow\tableheading{Method} & \tableheading{AIME24} & \tableheading{AIME25} & \tableheading{AMC23} & \tableheading{MATH500} & \tableheading{Minerva} & \tableheading{Olympiad} & \tableheading{Avg.}\\
\midrule
Teacher, published & 28.3 & 23.3 & 73.1 & 84.3 & 38.5 & 52.0 & 49.9\\
Teacher, reproduced & 35.4 & 27.9 & 80.9 & 87.9 & 44.9 & 56.5 & 55.6\\
\midrule
\sft{}, published & 22.1 & 22.9 & 63.8 & 80.8 & 36.0 & 44.7 & 45.0\\
OPD, published & 28.3 & 26.3 & 70.3 & 85.8 & 39.8 & 55.3 & 51.0\\
\midrule
\sft{} baseline, reproduced & 24.6 & 25.0 & 71.3 & 80.2 & 32.6 & 49.2 & 47.1\\
Tuned \sft{}, without RS & 34.6 & 29.2 & 77.5 & 86.2 & 42.5 & 53.9 & 54.0\\
Tuned \sft{}, RS & 34.6 & 30.4 & 80.0 & 86.0 & 44.1 & 52.6 & 54.6\\
OPD, reproduced & 37.1 & 27.1 & 81.6 & 86.6 & 39.2 & 56.7 & 54.7\\
\midrule
OPD $-$ RS \sft{} & +2.5 & -3.3 & +1.6 & +0.6 & -4.9 & +4.1 & +0.1\\
\bottomrule
\end{tabular}}
\end{table}

\section{Additional Related Work}
\label{app:related}
Sequence-level distillation trains students on teacher-generated outputs \citep{kim2016sequence}.
Later work enriches supervised training with synthetic instruction--response pairs, as in Self-Instruct \citep{wang2023selfinstruct}, and teacher-generated rationales, as in Distilling Step-by-Step \citep{hsieh2023stepbystep}.
Selecting the generated examples is also consequential: STaR and rejection-sampling fine-tuning retain reasoning traces that yield correct answers \citep{zelikman2022star,yuan2023rft}; RAFT selects examples by reward \citep{dong2023raft}; and ReST$^{\mathrm{EM}}$ alternates generation, filtering with binary feedback, and fine-tuning \citep{singh2024restem}.
The strength of a supervised baseline also depends on optimization: well-tuned baselines can outperform newer methods \citep{melis2018state}, where the rankings can change with the hyperparameter-search budget \citep{dodge2019show}. In this work, we examine how these data and tuning choices affect reported on-policy gains and whether on-policy supervision improves capability recovery for self-distillation setup.

\section{Model Merging Methods and Evaluation}
\label{app:merging}

We evaluate eleven merging methods using \textsc{MergeKit}
\citep{goddard-etal-2024-arcees}. \Cref{tab:merge-methods} summarizes the
hyperparameters explored across students and setups.
For each student and setup, we first screen methods at default coefficients, then tune
the leading methods' coefficients and sparsity settings. Hyperparameter search effort
varies across methods, and
\Cref{tab:merge-methods} lists all the search settings. \cref{tab:mia-merge,tab:agent3-merge}
report selected configurations and sweep results. Candidates are checked for tensor-count mismatches, nonfinite values, relative $L_2$
distance from the base, and malformed tool calls to filter out failed merged results.

The selected merge appears in the Merge row of \cref{app:per-model}, evaluated on the
full suite with three evaluation seeds and the same harness, hardware, and decoding
settings as the distillation arms (\cref{app:splits}). Merge selection uses these same validation sets as distillation checkpoints (\cref{app:selection}). Sweep tables report recovery of each expert's gain over the base
(\cref{eq:recovery-rate}), averaged equally across expert domains.

\begin{table}[!htbp]
\papertablestyle
\caption{\textbf{Merging hyperparameters.}
Singleton values denote fixed settings.
Per-expert coefficient configurations appear in \cref{tab:mia-merge,tab:agent3-merge}.}
\label{tab:merge-methods}
\centering
\footnotesize
\setlength{\tabcolsep}{6pt}
\begin{tabular}{ll}
\toprule
\tableheadrow\tableheading{Method} & \tableheading{Hyperparameters and values}\\
\midrule
Model soup \citep{wortsman2022soups} & Expert weights: $1/3$ each\\
Task arithmetic \citep{ilharco2023taskarithmetic} & Per-expert coefficient: $[0.3,2.0]$\\
TIES \citep{yadav2023ties} & Retained density: $\{.3,.5,.7,.75,.9\}$\\
DARE-TIES \citep{yu2024dare} & Retained density: $.5$\\
DARE-linear \citep{yu2024dare} & Per-expert retained density: $\{.6,1\}$\\
DELLA \citep{deep2024della} & Retained density: $.5$; drop-probability range: $.15$\\
Breadcrumbs \citep{davari2024breadcrumbs} & Retained density: $\{.5,.6\}$; outlier fraction: $\{.01,.02\}$\\
Breadcrumbs-TIES \citep{davari2024breadcrumbs} & Retained density: $\{.5,.9\}$; outlier fraction: $.01$\\
SCE \citep{wan2024fusechat} & Top-$k$: $100\%$\\
Model Stock \citep{jang2024modelstock} & Layer-wise weighting: off\\
Karcher mean \citep{karcher1977riemannian} & Iterations: 10; tolerance: $10^{-5}$\\
\midrule
Depth ramp (modifier) \citep{wang2024lines} & Shallow and deep coefficients, reported as their mean\\
\bottomrule
\end{tabular}
\end{table}

\begin{table}[!htbp]
\centering
\papertablestyle

\footnotesize
\setlength{\tabcolsep}{5pt}
\settowidth{\mergenumw}{103.6}
\renewcommand{\arraystretch}{1.06}
\caption{\textbf{Merge sweeps on \mia{}.}
\protect\mgSel{} marks the selected
merge (the Merge row of \cref{app:per-model}) and $^{\rho}$ a depth-ramped coefficient, all three when
set on the tuple; bold is each panel's best average, and gray averages cover the measured axes only.}\label{tab:mia-merge}
\begin{tabular}{cll>{\hspace{10pt}}rrr>{\hspace{12pt}}r}
\toprule
& & & \multicolumn{4}{c}{\tableheading{Recovery (\%)}}\\
\cmidrule(l{10pt}){4-7}
\tableheadrow & \tableheading{Method} & \tableheading{Coefficients} $(c_{\text{IF}},c_{\text{Math}},c_{\text{IC}})$ & \tableheading{IF} & \tableheading{Math} & \tableheading{IC-SQL} & \tableheading{Avg.}\\

\midrule
\mgPanel{\mainmodel{}}{all 8 configurations shown, 4 methods}
\mgHeadroom{29.5}{11.7}{16.1}
\mgSel & Task arith.\ & $(0.70,1.00,0.70)$ & 103.1 & 159.2 & 103.6 & \best{122.0}\\
& Task arith.\ & $(0.70,0.70,0.70)$ & 86.2 & 172.0 & 100.6 & 119.6\\
& SCE & top-$k$ $100\%$ & 63.0 & 176.4 & 103.6 & 114.3\\
& TIES ($d{=}.5$) & $(1.00,1.00,1.00)$ & 62.2 & 158.8 & 104.8 & 108.6\\
& Task arith.\ & $(0.70,1.00,0.70)^{\rho}$ & 91.4 & 121.4 & 103.6 & 105.5\\
& Task arith.\ & $(0.70,1.00,0.85)$ & 98.6 & 104.7 & 108.4 & 103.9\\
\addlinespace[2pt]
& Model soup & equal thirds & 46.5 & 144.1 & \mgNA & \mgGray{95.3}\\
& Task arith.\ & $(0.70,1.25,0.70)$ & 87.9 & 53.8 & \mgNA & \mgGray{70.9}\\
\addlinespace[4pt]
\mgPanel{\textsc{Qwen3-8B}}{10 of 13 configurations shown, 4 methods}
\mgHeadroom{4.3}{30.6}{15.9}
& Task arith.\ & $(0.70,1.30,0.70)$ & 156.5 & 82.9 & 103.7 & \best{114.4}\\
& Task arith.\ & $(0.70,0.70,0.70)$ & 152.2 & 73.9 & 101.2 & 109.1\\
& Task arith.\ & $(0.70,1.00,0.70)$ & 143.5 & 85.4 & 98.2 & 109.0\\
\mgSel & Task arith.\ & $(0.50,1.30,0.70)$ & 113.0 & 105.2 & 98.8 & 105.7\\
& Task arith.\ & $(1.00,0.70,0.70)$ & 152.2 & 62.4 & 98.2 & 104.3\\
& TIES ($d{=}.5$) & $(1.00,1.00,1.00)$ & 139.1 & 68.1 & 103.7 & 103.6\\
& TIES ($d{=}.5$) & $(0.70,1.60,0.70)$ & 121.7 & 92.4 & 90.9 & 101.7\\
& SCE & top-$k$ $100\%$ & 143.5 & 59.3 & 101.2 & 101.3\\
& Model soup & equal thirds & 56.5 & 41.0 & 63.4 & 53.6\\
\addlinespace[2pt]
& Task arith.\ & $(0.70,1.60,0.70)$ & 152.2 & \mgNeg{117.8} & \mgNA & \mgGray{17.2}\\
\addlinespace[4pt]
\mgPanel{\textsc{Qwen3-1.7B}}{10 of 18 configurations shown, 4 methods}
\mgHeadroom{34.3}{29.9}{40.8}
& SCE & top-$k$ $100\%$ & 73.0 & 53.8 & 103.6 & \best{76.8}\\
& TIES ($d{=}.7$) & $(1.30,1.30,0.70)$ & 77.8 & 54.1 & 98.1 & 76.7\\
& TIES ($d{=}.5$) & $(1.30,1.30,0.70)$ & 77.2 & 54.6 & 98.1 & 76.7\\
\mgSel & Task arith.\ & $(0.70,0.70,0.70)$ & 77.3 & 46.1 & 103.3 & 75.6\\
& TIES ($d{=}.5$) & $(1.00,1.00,1.00)$ & 67.1 & 54.8 & 103.6 & 75.2\\
& TIES ($d{=}.5$) & $(1.10,1.30,0.70)$ & 65.0 & 56.3 & 103.3 & 74.9\\
& Task arith.\ & $(0.85,0.85,0.85)$ & 102.5 & \mgNeg{3.0} & 97.9 & 65.8\\
& Model soup & equal thirds & 26.2 & 14.2 & 100.2 & 46.9\\
& Task arith.\ & $(1.00,1.00,1.00)$ & 107.7 & \mgNeg{100.5} & 95.0 & 34.1\\
& Task arith.\ & $(0.80,1.28,0.70)^{\rho}$ & 94.0 & \mgNeg{203.9} & 99.8 & \mgNeg{3.4}\\
\addlinespace[4pt]
\mgPanel{\textsc{Ministral-3-3B}}{10 of 54 configurations shown, 5 methods}
\mgHeadroom{27.1}{12.4}{47.7}
& TIES ($d{=}.5$) & $(1.50,0.55,0.55^{\rho})$ & 131.1 & 135.1 & 99.8 & \best{122.0}\\
& Task arith.\ & $(1.50,0.55,0.55^{\rho})$ & 131.2 & 134.4 & 100.2 & 121.9\\
& Task arith.\ & $(1.50,0.55,0.55)$ & 125.5 & 142.2 & 98.0 & 121.9\\
& DARE-TIES ($d{=}.5$) & $(1.50^{\rho},0.55,0.55^{\rho})$ & 132.6 & 102.6 & 97.6 & 110.9\\
\mgSel & TIES ($d{=}.7$) & $(1.50,0.55,0.55^{\rho})$ & 118.6 & 112.7 & 100.6 & 110.6\\
& Task arith.\ & $(0.70,0.70,0.70)$ & 83.6 & 128.5 & 100.0 & 104.0\\
& Model soup & equal thirds & 48.4 & 167.5 & 90.3 & 102.1\\
& Task arith.\ & $(1.00,1.00,1.00)$ & 87.0 & 64.3 & 98.6 & 83.3\\
& TIES ($d{=}.5$) & $(1.00,1.00,1.00)$ & \mgNeg{11.9} & 68.1 & 99.4 & 51.9\\
& SCE & top-$k$ $100\%$ & \mgNeg{23.0} & 28.8 & 99.6 & 35.2\\
\bottomrule
\end{tabular}
\end{table}

\begin{table}[!htbp]
\centering
\papertablestyle

\footnotesize
\setlength{\tabcolsep}{3.5pt}
\settowidth{\mergenumw}{103.6}
\renewcommand{\arraystretch}{1.06}
\caption{\textbf{Merge sweeps on \agentthree{}.}
\protect\mgSel{} marks the selected
merge (the Merge row of \cref{app:per-model}) and $^{\rho}$ a depth-ramped coefficient, all three when
set on the tuple; bold is each panel's best average, and gray averages cover the measured axes only.}\label{tab:agent3-merge}
\begin{tabular}{cll>{\hspace{6pt}}rrr>{\hspace{5pt}}r}
\toprule
& & & \multicolumn{4}{c}{\tableheading{Recovery (\%)}}\\
\cmidrule(l{7pt}){4-7}
\tableheadrow & \tableheading{Method} & \tableheading{Coefficients} $(c_{\text{IC}},c_{\text{ALF}},c_{\text{WS}})$ & \tableheading{IC-SQL} & \tableheading{ALFWorld} & \tableheading{WebShop} & \tableheading{Avg.}\\

\midrule
\mgPanel{\mainmodel{}}{10 of 32 configurations shown, 6 methods}
\mgHeadroom{15.7}{72.1}{74.4}
\mgSel & Task arith.\ & $(1.25,0.95,0.80)^{\rho}$ & 73.5 & 85.2 & 89.3 & \best{82.6}\\
& Task arith.\ & $(1.25,0.85,0.80)^{\rho}$ & 60.5 & 82.2 & 91.9 & 78.2\\
& Task arith.\ & $(1.25,0.85,0.80)$ & 37.0 & 84.2 & 90.9 & 70.7\\
& Task arith.\ & $(1.25,0.92,0.90)$ & 9.9 & 88.1 & 94.9 & 64.3\\
& Model soup & equal thirds & 71.6 & 36.6 & 16.9 & 41.7\\
& Task arith.\ & $(0.70,0.70,0.70)$ & \mgNeg{116.7} & 80.2 & 81.2 & 14.9\\
& SCE & top-$k$ $100\%$ & \mgNeg{122.2} & 75.3 & 76.3 & 9.8\\
& TIES ($d{=}.5$) & $(1.00,1.00,1.00)$ & \mgNeg{138.3} & 79.2 & 77.7 & 6.2\\
\addlinespace[2pt]
& Breadcrumbs ($d{=}.6$) & $(1.41,0.94,0.89)$ & \mgNA & 77.2 & 81.7 & \mgGray{79.5}\\
& DARE-linear ($d{=}.6/1$) & $(1.25,0.85,0.70)$ & \mgNA & 79.2 & 66.9 & \mgGray{73.1}\\
\addlinespace[4pt]
\mgPanel{\textsc{Qwen3-8B}}{all 6 configurations shown, 3 methods}
\mgHeadroom{15.6}{55.7}{69.8}
\mgSel & Task arith.\ & $(1.25,0.95,0.80)$ & 106.2 & 111.5 & 100.9 & \best{106.2}\\
& Task arith.\ & $(1.25,0.95,0.65)$ & 100.0 & 109.0 & 99.1 & 102.7\\
& Task arith.\ & $(1.25,1.05,0.65)$ & 101.3 & 105.1 & 95.1 & 100.5\\
& TIES ($d{=}.5$) & $(1.00,1.00,1.00)$ & 91.3 & 89.7 & 101.2 & 94.1\\
& Task arith.\ & $(1.25,0.95,0.50)$ & 102.5 & 106.4 & 63.3 & 90.8\\
& Model soup & equal thirds & 63.4 & 26.9 & 59.9 & 50.1\\
\addlinespace[4pt]
\mgPanel{\textsc{Qwen3-1.7B}}{all 10 configurations shown, 3 methods}
\mgHeadroom{40.8}{82.1}{18.4}
\mgSel & Task arith.\ & $(0.30,0.55,0.90)^{\rho}$ & 113.5 & 92.2 & 96.7 & \best{100.8}\\
& Task arith.\ & $(1.25,0.55,0.80)$ & 120.9 & 91.3 & 66.3 & 92.8\\
& Task arith.\ & $(0.30,0.55,0.90)$ & 107.4 & 92.2 & 75.0 & 91.5\\
& Task arith.\ & $(0.70,0.70,0.70)$ & 117.5 & 95.7 & 55.4 & 89.5\\
& Task arith.\ & $(0.80,0.30,0.90)$ & 104.0 & 33.9 & 129.4 & 89.1\\
& Task arith.\ & $(0.30,0.33,0.90)$ & 102.8 & 38.3 & 116.3 & 85.8\\
& Task arith.\ & $(1.25,0.95,0.80)$ & 120.6 & 96.5 & 33.7 & 83.6\\
& TIES ($d{=}.5$) & $(1.00,1.00,1.00)$ & 107.4 & 96.5 & 17.4 & 73.8\\
& Task arith.\ & $(1.25,0.30,1.50)$ & 87.2 & 28.7 & 98.9 & 71.6\\
& Model soup & equal thirds & 105.0 & 36.5 & 41.3 & 60.9\\
\addlinespace[4pt]
\mgPanel{\textsc{Ministral-3-3B}}{11 of 64 configurations shown, 9 methods}
\mgHeadroom{33.2}{60.7}{62.8}
& Breadcrumbs-TIES ($d{=}.5$) & $(1.00,0.43,0.61)$ & 85.7 & 58.8 & 83.8 & \best{76.1}\\
& Task arith.\ & $(1.00,0.43,0.60)$ & 67.9 & 75.3 & 77.1 & 73.4\\
& TIES ($d{=}.5$) & $(1.00,0.43,0.61)$ & 75.5 & 63.5 & 79.6 & 72.9\\
\mgSel & Task arith.\ & $(1.00,0.43,0.61)$ & 62.7 & 76.5 & 78.0 & 72.4\\
& Model soup & equal thirds & 35.0 & 80.0 & 50.3 & 55.1\\
& DARE-TIES ($d{=}.5$) & $(1.00,0.43,0.61)$ & 20.1 & 77.7 & 48.7 & 48.8\\
& DELLA ($d{=}.5$) & $(1.00,0.43,0.61)$ & 8.2 & 75.3 & 46.2 & 43.2\\
& Task arith.\ & $(0.70,0.70,0.70)$ & \mgNeg{82.2} & 98.8 & 15.0 & 10.5\\
\addlinespace[2pt]
& Karcher mean & fixed $1/K$ & \mgNA & \mgNA & 41.7 & \mgGray{41.7}\\
& SCE & top-$k$ $100\%$ & \mgNA & \mgNA & 38.2 & \mgGray{38.2}\\
& Model Stock & cosine-derived & \mgNA & \mgNA & 0.3 & \mgGray{0.3}\\
\bottomrule
\end{tabular}
\end{table}

\section{Experimental Details}
\label{app:experimental}

\subsection{Teachers and Distillation Corpus}
\label{app:corpus}
In the main self-distillation experiments, teachers share their student's initialization, tokenizer, and chat template. Each teacher is RL-trained from that initialization, and we evaluate the selected teacher checkpoints using the same evaluation harness as the students. For each student and setup, we construct one rejection-sampled teacher corpus \citep{yuan2023rft}: we sample teacher trajectories and retain at most one trajectory with verified reward $1$ per prompt. Generation costs are reported in \cref{app:teacher-generation}. SFT, Soft-KD, and Offline-MOPD use the recorded prompt--trajectory pairs; MOPD uses the same prompts and generates its own trajectories during training. We verify that the prompt sets match by comparing their hashes.
All four methods route prompts to teachers by the known task domain, so the comparison measures capability integration independently of router quality.

\subsection{Training}
\label{app:training-config}
\label{app:method-choices}
We train SFT and Soft-KD with TRL \citep{vonwerra2020trl}. Offline-MOPD, MOPD, and RL teacher training use \textsc{verl} \citep{sheng2025hybridflow}, with \textsc{verl-agent} for multi-turn tasks \citep{feng2025gigpo}. We use vLLM to serve student rollouts and teacher scoring in the two MOPD variants \citep{kwon2023vllm}.
MOPD uses one rollout per prompt and optimizes per-token reverse KL (\cref{eq:reverse-kd}). Offline-MOPD uses the same trainer as MOPD but replaces live environment interaction with replay of recorded teacher observations: in \agentthree{}, MOPD executes student actions to obtain fresh environment observations, whereas Offline-MOPD conditions each student response on the recorded teacher history. For single-turn tasks, there is no interaction history to replay, so Offline-MOPD samples a complete student response from the original prompt and applies the same reverse-KL objective as MOPD.

\FloatBarrier
\subsection{Checkpoint Selection and Evaluation}
\label{app:selection}
\Cref{tab:hparams} lists the selected learning rates and batch sizes. Each distillation run starts from the base checkpoint and trains for five epochs over the shared prompts. We use three learning rates per grid, re-center each grid on the nearest previously selected rate for that model and setup, and extend it when the best rate lies at an endpoint. Training and evaluation use the same prompt format, stop tokens, and prompt and response length limits. We choose the evaluated checkpoint with the highest unweighted mean validation score across the setup's domains. Validation uses model-generated rollouts under the same held-out evaluation harness.
\Cref{tab:splits} defines the training, validation, and test splits for each domain; \cref{tab:suites} lists the evaluation benchmarks, metrics, and test-set sizes. We split InterCode-SQL by Spider database \citep{yu2018spider} to prevent queries sharing a schema across splits. For mathematics and instruction following, we use the teachers' held-out RL validation splits after removing prompts in the student's training corpus.

\begin{table}[H]
\papertablestyle
\caption{\textbf{Training configuration.}
Selected learning rates differ across students and setups, so each default grid is re-centered for every student and setup as described in the text. }
\label{tab:hparams}
\footnotesize
\centering

\vspace{6pt}
\setlength{\tabcolsep}{2pt}
\begin{tabularx}{\linewidth}{l*{8}{C}}
\toprule
 & \multicolumn{2}{c}{\tableheading{\textsc{Qwen3-1.7B}}} & \multicolumn{2}{c}{\tableheading{\textsc{Ministral-3-3B}}} & \multicolumn{2}{c}{\tableheading{\textsc{Qwen3.5-4B}}} & \multicolumn{2}{c}{\tableheading{\textsc{Qwen3-8B}}}\\
\cmidrule(lr){2-3}\cmidrule(lr){4-5}\cmidrule(lr){6-7}\cmidrule(lr){8-9}
\tableheadrow\tableheading{Method} & \tableheading{\mia{}} & \tableheading{A$\times$3} & \tableheading{\mia{}} & \tableheading{A$\times$3} & \tableheading{\mia{}} & \tableheading{A$\times$3} & \tableheading{\mia{}} & \tableheading{A$\times$3}\\
\midrule
\multicolumn{9}{l}{\textit{Selected learning rate ($\times10^{-6}$)}}\\
SFT & 20 & 10 & 2 & 5 & 10 & 5 & 5 & 5\\
Soft-KD & 5 & 6 & 0.63 & 10 & 16 & 16 & 6.3 & 5\\
Off-MOPD & 1 & 2 & 0.2 & 0.95 & 1 & 3.2 & 1 & 3\\
MOPD & 1.5 & 1  & 0.63 & 0.95 & 1 & 3.2 & 3 & 3\\
\addlinespace[3pt]
\multicolumn{9}{l}{\textit{Batch size}}\\
SFT & 32 & 32 & 32 & 32 & 32 & 32 & 32 & 32\\
Soft-KD & 32 & 32 & 32 & 32 & 32 & 32 & 32 & 32\\
Off-MOPD & 24 & 32 & 24 & 16 & 16 & 16 & 24 & 40\\
MOPD & 24 & 16 & 24 & 16 & 48 & 16 & 96 & 40\\
\bottomrule
\end{tabularx}

\end{table}

\subsection{Task Splits and Benchmarks}
\label{app:splits}
\Cref{tab:splits} defines the training, validation, and test splits for each domain; \cref{tab:suites} lists the evaluation benchmarks, metrics, and test-set sizes. We split InterCode-SQL by Spider database \citep{yu2018spider} to prevent queries sharing a schema across splits. For mathematics and instruction following, we use the teachers' held-out RL validation splits after removing prompts in the student's training corpus. Holding out part of the distillation corpus would not ensure that validation is also held out from teacher training: in our check, an instruction-following subset constructed this way lay entirely within the teacher's RL training set and inflated the teacher ceiling by over $40$ points, thereby inflating the recovery denominator in \cref{eq:recovery-rate}.

\begin{table}[H]
\papertablestyle
\caption{\textbf{Task splits.}
Training, validation, and test splits are disjoint and shared by all methods for each student and setup. Validation excludes both teacher and student training data.}
\label{tab:splits}
\footnotesize
\centering
\footnotesize
\setlength{\tabcolsep}{4pt}
\begin{tabularx}{\linewidth}{lYYY}
\toprule
\tableheadrow\tableheading{Domain} & \tableheading{Teacher RL training} & \tableheading{Validation} & \tableheading{Test}\\
\midrule
Mathematics & DAPO-Math-17k & teacher's held-out RL split (512) & MATH-500, GSM8K, AIME24, AIME25\\
\midrule
Instruction following & IF-RLVR multi-constraint & teacher's held-out RL split (430) & IFEval, IFBench\\
\midrule
InterCode-SQL & 124 Spider train DBs & 16 held-out Spider train DBs (815) & Spider dev (1,034)\\
\midrule
ALFWorld & \texttt{train} & \texttt{valid\_train} (200) & \texttt{valid\_seen} (140), \texttt{valid\_unseen} (134)\\
\midrule
WebShop & goals $\geq$1,500 & goals 500--1,499 (1,000) & goals 0--499 (500)\\
\arrayrulecolor{black}\bottomrule
\end{tabularx}

\end{table}

\begin{table}[H]
\centering
\papertablestyle

\footnotesize
\setlength{\tabcolsep}{4pt}
\caption{\textbf{Benchmarks in the unified evaluation.}
Benchmarks, metrics, and test-set sizes for \cref{tab:capability-recovery,app:per-model}. A benchmark is ID if its domain is covered by the setup's teachers and OOD otherwise, and every benchmark counts toward that setup's recovery or retention. Both setups use the same evaluation suite.}\label{tab:suites}
\begin{tabularx}{\linewidth}{lYlccc}
\toprule
 &  &  &  & \multicolumn{2}{c}{\tableheading{Partition}}\\
\cmidrule(lr){5-6}
\tableheadrow\tableheading{Benchmark} & \tableheading{Domain} & \tableheading{Metric} & \tableheading{$n$} & \tableheading{\mia{}} & \tableheading{\agentthree{}}\\

\midrule
IFEval & instruction following & prompt-strict accuracy & 541 & ID & OOD\\
IFBench & instruction following & prompt-strict accuracy & 300 & ID & OOD\\
MATH-500 & mathematics & exact-match accuracy & 500 & ID & OOD\\
GSM8K & mathematics & exact-match accuracy & 1,319 & ID & OOD\\
AIME24 & mathematics & avg@16 accuracy ($T{=}0.6$) & 30 & ID & OOD\\
AIME25 & mathematics & avg@16 accuracy ($T{=}0.6$) & 30 & ID & OOD\\
InterCode-SQL & coding agent & task success rate & 1,034 & ID & ID\\
ALFWorld & household agent & success rate & 274 & OOD & ID\\
WebShop & shopping agent & success rate & 500 & OOD & ID\\
GPQA-Diamond & none & accuracy & 198 & OOD & OOD\\
MMLU-Pro & none & accuracy & 12,032 & OOD & OOD\\
\midrule
\multicolumn{4}{l}{\itshape Benchmarks per partition} & 7\,/\,4 & 3\,/\,8\\
\bottomrule
\end{tabularx}
\end{table}

\section{Why Off-Policy Distillation Can Match On-Policy Distillation}
\label{app:offpolicy-theory}

Motivated by these empirical results, we analyze when off-policy distillation can remain competitive
despite the distribution mismatch between teacher-generated training trajectories and student rollouts.

\subsection{Exposure Bias and Prior Work}
\label{app:offpolicy-background}

Teacher-generated trajectories have long been used to distill multiple RL policies into a single student \citep{rusu2016policy}.
When teacher and student share a base model, SFT can match RL expert performance without a clear advantage from on-policy distillation \citep{zhang2026online}.
Agentic studies similarly report successful single-task transfer, but observe degradation when combining tasks \citep{wang2026twophase}.
Our multi-teacher experiments show that tuned off-policy methods can still recover expert capabilities comparably to MOPD.

However, the concern about exposure bias remains:
off-policy training uses teacher-generated trajectories, but student actions determine the states encountered during generation.
Early errors can shift the later trajectory away from the training distribution.
We therefore analyze whether low token loss on teacher trajectories keeps student and teacher rollout distributions close despite
exposure bias.

\subsection{Distillation Losses and Success Probabilities}
\label{app:offpolicy-result}
Fix a domain $d$ and a prompt. Let $P_T$ and $P_\theta$ denote the teacher and student trajectory distributions.
Assume a shared environment and stopping rule, bounded trajectory lengths, and finite expected losses.
Both policies sample softmax outputs at temperature $1$.
Let $h_k$ precede assistant turn $k$, with at most $N$ turns per trajectory.
Write $Q_T$ and $Q_\theta$ for the teacher and student distributions over complete responses at a given history.
The unfiltered population losses for the forward-KL component of Soft-KD, MOPD, and Offline-MOPD are
\begin{equation}
\begin{aligned}
L_{\mathrm{off}}
&=\E_{\tau\sim P_T}\sum_{t=1}^{|\tau|}\ell_{\mathrm{FKL}}(\theta;d,s_t),\\
L_{\mathrm{on}}
&=\E_{\tau\sim P_\theta}\sum_{t=1}^{|\tau|}\ell_{\mathrm{RKL}}(\theta;d,s_t),\\
L_{\mathrm{hyb}}
&=\E_{\tau\sim P_T}\sum_k\KL\bigl(Q_\theta(\cdot\mid h_k)\,\|\,Q_T(\cdot\mid h_k)\bigr).
\end{aligned}
\label{eq:seq-objectives}
\end{equation}
Token losses are summed without length normalization, using the notation of \cref{eq:mopd}.
The KL chain rule identifies each response-level KL with the expected sum of reverse-KL token losses along a student response.

\begin{proposition}
\label{prop:offpolicy}
The trajectory losses satisfy
\[
L_{\mathrm{off}}=\KL(P_T\,\|\,P_\theta),
\qquad
L_{\mathrm{on}}=\KL(P_\theta\,\|\,P_T).
\]
For every trajectory statistic $F$ taking values in $[0,1]$,
\[
\bigl|\E_{P_T}F-\E_{P_\theta}F\bigr|
\le \operatorname{TV}(P_T,P_\theta)
\le \sqrt{\frac{\min\{L_{\mathrm{off}},\,L_{\mathrm{on}},\,N L_{\mathrm{hyb}}\}}{2}},
\]
where $\operatorname{TV}$ denotes total-variation distance.
At a context $s$, suppose that for some $m\ge0$,
\[
e^{-m}\le\frac{\pi_d(a\mid s)}{\pi_\theta(a\mid s)}\le e^m
\qquad\text{for every token }a.
\]
Then
\[
e^{-m}\ell_{\mathrm{RKL}}(\theta;d,s)
\le\ell_{\mathrm{FKL}}(\theta;d,s)
\le e^m\ell_{\mathrm{RKL}}(\theta;d,s).
\]
\end{proposition}

\noindent\textit{Proof.}
The KL chain rule gives the trajectory identities: the shared environment contributes zero conditional KL, leaving only assistant-token losses.
For a statistic in $[0,1]$, its expectation is the integral of its upper-tail probabilities.
The expectation gap is therefore bounded by total variation.
Pinsker's inequality then gives the bounds involving $L_{\mathrm{off}}$ and $L_{\mathrm{on}}$ \citep{tsybakov2009introduction}.

For $L_{\mathrm{hyb}}$, let $\delta_k$ and $D_k$ be the total variation and reverse KL between the student and teacher response distributions at $h_k$, setting both to zero after termination.
Replacing teacher responses with student responses from last to first preserves teacher histories before each replacement and a common student continuation afterward.
The triangle inequality, Pinsker's inequality, and Cauchy--Schwarz give
\[
\operatorname{TV}(P_T,P_\theta)
\le\E_{P_T}\sum_k\delta_k
\le\E_{P_T}\sum_k\sqrt{D_k/2}
\le\sqrt{N L_{\mathrm{hyb}}/2}.
\]

For the token-loss comparison, let $r$ be the teacher-to-student probability ratio and consider
\[
\begin{gathered}
f(u)=u\log u-u+1,
\qquad
g(u)=u-1-\log u,\\
f''(u)=u\,g''(u).
\end{gathered}
\]
Since $r$ has mean one under the student, the forward and reverse KL losses are the student expectations of $f(r)$ and $g(r)$, respectively.
Both functions and their first derivatives vanish at $1$.
Integrating the second-derivative bounds twice from $1$ yields
\[
e^{-m}g(u)\le f(u)\le e^m g(u),
\qquad u\in[e^{-m},e^m].
\]
Taking student expectations completes the proof.\hfill$\square$

Choosing $F$ as the task-success indicator bounds the success gap; the bounds using $L_{\mathrm{off}}$ and $L_{\mathrm{on}}$ introduce no trajectory-length factor beyond the summed losses.
The expected summed SFT loss under $P_T$ equals $L_{\mathrm{off}}$ plus a teacher-dependent entropy term.

For prompts with positive teacher success probability, rejection sampling conditions the teacher distribution on success (\cref{app:corpus}).
The expected summed SFT loss then equals forward trajectory KL from this conditional distribution to the student, up to a parameter-independent constant.
Soft-KD and Offline-MOPD retain the original teacher targets on successful histories, so the unfiltered bounds do not directly apply to their filtered losses.
In self-distillation, each teacher is realizable by the student, so all three losses can vanish on an individual domain.
A simultaneous fit is not guaranteed; empirically, Mix-RL recovers most teachers' gains in one student (\cref{sec:merge}).

\section{Per-Model Capability Recovery}
\label{app:per-model}
This appendix breaks down \cref{tab:aggregate-models,tab:scale} by benchmark, in the format of \cref{tab:capability-recovery}, for the other three students (\cref{tab:recovery-qwen3-1p7b,tab:recovery-ministral-3b,tab:recovery-qwen35-4b-rerun}) and for \textsc{Qwen3-1.7B} with 8B teachers (\cref{tab:recovery-qwen3-1p7b-8bteacher}).
Under self-distillation, neither in-domain recovery nor out-of-domain retention consistently favors on-policy distillation: for every student, the four distillation methods recover within 3.5 percentage points of one another on \mia{} and 4.6 on \agentthree{}, and MOPD improves mean OOD retention over SFT by 1.1 benchmark points (\cref{sec:main}).
With 8B teachers the methods separate: SFT and Soft-KD recover more in domain than both MOPD variants but lose more out of domain, most sharply on \agentthree{} (\cref{sec:teacher-scale}).

\begin{table}[H]
\papertablestyle
\caption{\textbf{Per-benchmark results for \textsc{Qwen3-1.7B}.}
Notation follows \cref{tab:capability-recovery}.
ID Rec.\ and OOD~$\Delta$ pool ALFWorld's seen (ALF-s) and unseen (ALF-u) splits. }
\label{tab:recovery-qwen3-1p7b}
\footnotesize
\centering
\begingroup
\setlength{\parskip}{0pt}
\fontsize{6.2}{7.4}\selectfont
\setlength{\tabcolsep}{0.8pt}
\begin{tabularx}{\linewidth}{l*{12}{>{\hsize=0.99\hsize\linewidth=\hsize}R}>{\hsize=1.12\hsize\linewidth=\hsize}R R}
\toprule
 & \multicolumn{8}{c}{\textbf{Single-turn}} & \multicolumn{4}{c}{\textbf{Multi-turn}} & \multicolumn{2}{c}{\textbf{Overall}}\\
\cmidrule(lr){2-9}\cmidrule(lr){10-13}\cmidrule(l){14-15}
\tableheadrow\textbf{\mia{}} & \recoveryID{MATH} & \recoveryID{GSM} & \recoveryID{AIME24} & \recoveryID{AIME25} & \recoveryID{IFEval} & \recoveryID{IFBen} & \recoveryOOD{GPQA} & \recoveryOOD{MMLU} & \recoveryID{IC-SQL} & \recoveryOOD{ALF-s} & \recoveryOOD{ALF-u} & \recoveryOOD{WShop} & \colhead{ID Rec.} & \colhead{OOD $\Delta$}\\
\midrule
Base & \reprosim{73.4}{0.9} & \reprosim{82.7}{0.1} & \reprosim{12.9}{0.4} & \reprosim{11.2}{0.8} & \reprosim{66.7}{0.1} & \reprosim{18.3}{0.3} & \reprosim{30.8}{0.5} & \reprosim{47.0}{0.1} & \reprosim{12.4}{0.1} & \reprosim{9.3}{0.7} & \reprosim{9.0}{0.7} & \reprosim{6.6}{0.2} & \reproref{\na} & \reproref{0.0}\\
Teachers & \reprosim{79.6}{0.4} & \reprosim{85.7}{0.2} & \reprosim{21.5}{1.4} & \reprosim{19.4}{1.1} & \reprosim{83.4}{0.1} & \reprosim{42.3}{0.3} & \reprosim{31.3}{0.5} & \reprosim{46.8}{0.1} & \reprosim{61.6}{0.1} & \reprosim{7.1}{0.7} & \reprosim{6.7}{0.7} & \reprosim{4.4}{0.2} & \reproref{100.0} & \reproref{\na}\\
\midrule[.3pt]
Mix-RL & \reprosim{79.7}{0.3} & \reprosim{84.3}{0.6} & \reprosim{20.0}{1.9} & \reprosim{20.4}{1.9} & \reprosim{83.5}{0.4} & \reprosim{41.6}{1.4} & \reprosim{29.3}{0.9} & \reprosim{40.1}{0.4} & \reprosim{55.1}{0.7} & \reprosim{5.7}{1.4} & \reprosim{1.2}{1.6} & \reprosim{1.8}{0.5} & \reprosim{92.3}{1.3} & \reprosim{$-$4.7}{0.3}\\
Merge & \reprosim{78.3}{0.3} & \reprosim{83.3}{0.2} & \reprosim{19.9}{0.6} & \reprosim{16.5}{0.8} & \reprosim{84.2}{0.1} & \reprosim{41.3}{0.3} & \reprosim{29.3}{0.5} & \reprosim{40.7}{0.1} & \reprosim{54.0}{0.1} & \reprosim{7.1}{0.7} & \reprosim{3.7}{0.7} & \reproref{0.0} & \reprosim{86.2}{0.2} & \reprosim{$-$4.5}{0.2}\\
\midrule[.3pt]
SFT & \reprosim{78.9}{0.6} & \reprosim{85.2}{0.5} & \reprosim{18.6}{1.3} & \reprosim{17.8}{2.0} & \reprosim{\best{84.5}}{0.2} & \reprosim{41.2}{0.7} & \reprosim{26.3}{0.9} & \reprosim{\best{41.1}}{0.2} & \reprosim{\best{61.2}}{0.4} & \reprosim{1.2}{1.1} & \reprosim{0.7}{0.7} & \reproref{0.0} & \reprosim{94.7}{2.9} & \reprosim{$-$6.3}{0.3}\\
Soft-KD & \reprosim{79.6}{0.5} & \reprosim{\best{85.5}}{0.2} & \reprosim{\best{20.6}}{1.0} & \reprosim{20.1}{0.7} & \reprosim{83.7}{1.1} & \reprosim{41.3}{0.6} & \reprosim{26.4}{0.8} & \reprosim{36.4}{0.6} & \reprosim{60.6}{0.3} & \reproref{0.0} & \reproref{0.0} & \reproref{0.0} & \reprosim{\best{98.2}}{2.0} & \reprosim{$-$7.7}{0.2}\\
Off-MOPD & \reprosim{\best{82.0}}{0.7} & \reprosim{83.9}{0.5} & \reprosim{19.6}{1.8} & \reprosim{19.3}{1.0} & \reprosim{82.6}{0.7} & \reprosim{\best{42.6}}{1.8} & \reprosim{\best{26.9}}{1.1} & \reprosim{36.1}{0.9} & \reprosim{60.8}{1.4} & \reprosim{0.7}{0.7} & \reproref{0.0} & \reprosim{2.7}{0.6} & \reprosim{97.7}{2.3} & \reprosim{$-$6.9}{0.5}\\
MOPD & \reprosim{79.3}{1.0} & \reprosim{84.0}{0.7} & \reprosim{19.4}{0.6} & \reprosim{\best{21.5}}{1.2} & \reprosim{83.9}{0.8} & \reprosim{41.8}{0.8} & \reprosim{25.6}{1.3} & \reprosim{36.8}{0.8} & \reprosim{60.0}{0.4} & \reprosim{\best{4.3}}{2.1} & \reprosim{\best{3.0}}{0.7} & \reprosim{\best{3.2}}{1.0} & \reprosim{96.9}{1.8} & \reprosim{\best{$-$6.1}}{0.8}\\
\addlinespace[3pt]
\midrule
\tableheadrow\textbf{\boldmath\agentthree{}} & \recoveryOOD{MATH} & \recoveryOOD{GSM} & \recoveryOOD{AIME24} & \recoveryOOD{AIME25} & \recoveryOOD{IFEval} & \recoveryOOD{IFBen} & \recoveryOOD{GPQA} & \recoveryOOD{MMLU} & \recoveryID{IC-SQL} & \recoveryID{ALF-s} & \recoveryID{ALF-u} & \recoveryID{WShop} & \colhead{ID Rec.} & \colhead{OOD $\Delta$}\\
\midrule
Base & \reprosim{73.4}{0.9} & \reprosim{82.7}{0.1} & \reprosim{12.9}{0.4} & \reprosim{11.2}{0.8} & \reprosim{66.7}{0.1} & \reprosim{18.3}{0.3} & \reprosim{30.8}{0.5} & \reprosim{47.0}{0.1} & \reprosim{12.4}{0.1} & \reprosim{9.3}{0.7} & \reprosim{9.0}{0.7} & \reprosim{6.6}{0.2} & \reproref{\na} & \reproref{0.0}\\
Teachers & \reprosim{75.6}{0.4} & \reprosim{83.2}{0.2} & \reprosim{15.0}{0.4} & \reprosim{11.0}{0.8} & \reprosim{67.5}{0.1} & \reprosim{19.7}{0.3} & \reprosim{31.3}{0.5} & \reprosim{47.3}{0.1} & \reprosim{61.6}{0.1} & \reprosim{88.6}{0.7} & \reprosim{85.1}{0.7} & \reprosim{65.8}{0.5} & \reproref{100.0} & \reproref{\na}\\
\midrule[.3pt]
Mix-RL & \reprosim{71.2}{0.2} & \reprosim{83.4}{0.3} & \reprosim{10.5}{1.3} & \reprosim{10.3}{0.6} & \reprosim{67.5}{0.4} & \reprosim{16.1}{0.5} & \reprosim{30.8}{0.9} & \reprosim{43.6}{0.2} & \reprosim{62.0}{2.0} & \reprosim{90.0}{1.4} & \reprosim{88.1}{2.0} & \reprosim{66.8}{0.9} & \reprosim{101.9}{0.8} & \reprosim{$-$1.2}{0.1}\\
Merge & \reprosim{73.8}{1.1} & \reprosim{81.4}{0.1} & \reprosim{11.2}{2.0} & \reprosim{13.1}{1.6} & \reprosim{63.9}{0.2} & \reprosim{19.1}{0.4} & \reprosim{28.3}{0.5} & \reprosim{45.9}{0.1} & \reprosim{53.6}{0.1} & \reprosim{82.9}{0.7} & \reprosim{80.6}{0.7} & \reprosim{57.2}{0.2} & \reprosim{88.3}{0.3} & \reprosim{$-$0.8}{0.2}\\
\midrule[.3pt]
SFT & \reprosim{68.9}{1.5} & \reprosim{77.2}{0.5} & \reprosim{9.7}{3.3} & \reprosim{8.6}{2.2} & \reprosim{55.3}{0.8} & \reprosim{14.9}{0.4} & \reprosim{24.4}{0.6} & \reprosim{40.9}{0.3} & \reprosim{\best{60.9}}{0.6} & \reprosim{87.4}{1.1} & \reprosim{84.3}{0.7} & \reprosim{\best{64.4}}{1.2} & \reprosim{\best{98.3}}{0.7} & \reprosim{$-$5.4}{0.9}\\
Soft-KD & \reprosim{\best{70.8}}{0.4} & \reprosim{80.3}{0.6} & \reprosim{9.6}{1.3} & \reprosim{10.8}{0.6} & \reprosim{56.5}{1.2} & \reprosim{16.4}{0.8} & \reprosim{22.6}{1.1} & \reprosim{43.1}{0.3} & \reprosim{60.5}{0.2} & \reprosim{85.7}{2.1} & \reprosim{82.8}{1.5} & \reprosim{63.6}{1.8} & \reprosim{96.8}{1.7} & \reprosim{$-$4.1}{0.5}\\
Off-MOPD & \reprosim{67.7}{0.9} & \reprosim{\best{82.5}}{0.1} & \reprosim{\best{12.0}}{1.3} & \reprosim{\best{11.8}}{1.5} & \reprosim{63.3}{0.9} & \reprosim{\best{16.7}}{0.6} & \reprosim{24.4}{0.8} & \reprosim{43.8}{0.2} & \reprosim{59.0}{1.6} & \reprosim{86.2}{1.1} & \reprosim{82.1}{2.2} & \reprosim{64.2}{1.7} & \reprosim{96.3}{2.1} & \reprosim{$-$2.6}{0.2}\\
MOPD & \reprosim{70.2}{0.5} & \reprosim{82.2}{0.5} & \reprosim{11.4}{1.9} & \reprosim{11.0}{1.2} & \reprosim{\best{63.9}}{0.9} & \reprosim{16.4}{0.7} & \reprosim{\best{26.1}}{1.1} & \reprosim{\best{45.2}}{0.2} & \reprosim{57.1}{1.8} & \reprosim{\best{88.6}}{1.2} & \reprosim{\best{85.8}}{2.0} & \reprosim{64.0}{0.9} & \reprosim{96.8}{0.6} & \reprosim{\best{$-$2.1}}{0.5}\\
\bottomrule
\end{tabularx}
\endgroup
\end{table}

\begin{table}[H]
\papertablestyle
\caption{\textbf{Per-benchmark results for \textsc{Ministral-3-3B}.}
Notation follows \cref{tab:capability-recovery}.
ID Rec.\ and OOD~$\Delta$ pool ALFWorld's seen (ALF-s) and unseen (ALF-u) splits. }
\label{tab:recovery-ministral-3b}
\footnotesize
\centering
\begingroup
\setlength{\parskip}{0pt}
\fontsize{6.2}{7.4}\selectfont
\setlength{\tabcolsep}{0.8pt}
\begin{tabularx}{\linewidth}{l*{12}{>{\hsize=0.99\hsize\linewidth=\hsize}R}>{\hsize=1.12\hsize\linewidth=\hsize}R R}
\toprule
 & \multicolumn{8}{c}{\textbf{Single-turn}} & \multicolumn{4}{c}{\textbf{Multi-turn}} & \multicolumn{2}{c}{\textbf{Overall}}\\
\cmidrule(lr){2-9}\cmidrule(lr){10-13}\cmidrule(l){14-15}
\tableheadrow\textbf{\mia{}} & \recoveryID{MATH} & \recoveryID{GSM} & \recoveryID{AIME24} & \recoveryID{AIME25} & \recoveryID{IFEval} & \recoveryID{IFBen} & \recoveryOOD{GPQA} & \recoveryOOD{MMLU} & \recoveryID{IC-SQL} & \recoveryOOD{ALF-s} & \recoveryOOD{ALF-u} & \recoveryOOD{WShop} & \colhead{ID Rec.} & \colhead{OOD $\Delta$}\\
\midrule
Base & \reprosim{85.4}{0.3} & \reprosim{91.5}{0.0} & \reprosim{23.5}{0.3} & \reprosim{21.9}{0.4} & \reprosim{56.0}{0.2} & \reprosim{17.7}{0.3} & \reprosim{41.4}{0.5} & \reprosim{60.1}{0.0} & \reprosim{40.4}{0.1} & \reprosim{6.4}{0.7} & \reprosim{9.0}{0.7} & \reprosim{0.4}{0.2} & \reproref{\na} & \reproref{0.0}\\
Teachers & \reprosim{86.6}{0.2} & \reprosim{90.4}{0.2} & \reprosim{26.2}{0.3} & \reprosim{23.1}{0.7} & \reprosim{69.1}{0.1} & \reprosim{24.3}{0.3} & \reprosim{39.9}{0.5} & \reprosim{59.0}{0.0} & \reprosim{74.1}{0.1} & \reprosim{23.6}{1.2} & \reprosim{20.1}{0.7} & \reprosim{0.4}{0.2} & \reproref{100.0} & \reproref{\na}\\
\midrule[.3pt]
Mix-RL & \reprosim{82.6}{0.9} & \reprosim{89.6}{0.5} & \reprosim{24.3}{1.4} & \reprosim{22.5}{1.0} & \reprosim{70.5}{0.4} & \reprosim{26.0}{1.2} & \reprosim{39.2}{0.6} & \reprosim{46.2}{0.9} & \reprosim{74.4}{0.3} & \reprosim{15.0}{2.6} & \reprosim{13.4}{2.0} & \reprosim{1.2}{0.3} & \reprosim{93.2}{4.8} & \reprosim{$-$2.2}{0.6}\\
Merge & \reprosim{82.1}{0.3} & \reprosim{89.6}{0.1} & \reprosim{21.6}{0.6} & \reprosim{21.0}{0.4} & \reprosim{75.7}{0.2} & \reprosim{25.8}{0.4} & \reprosim{36.2}{0.6} & \reprosim{48.3}{0.0} & \reprosim{77.8}{0.1} & \reprosim{10.5}{0.8} & \reprosim{16.2}{0.9} & \reprosim{0.9}{0.2} & \reprosim{99.7}{1.7} & \reprosim{$-$2.7}{0.1}\\
\midrule[.3pt]
SFT & \reprosim{83.0}{0.4} & \reprosim{88.5}{0.3} & \reprosim{24.7}{0.5} & \reprosim{\best{24.0}}{0.8} & \reprosim{69.9}{0.3} & \reprosim{26.1}{0.4} & \reprosim{\best{38.4}}{1.0} & \reprosim{49.5}{1.1} & \reprosim{75.0}{0.7} & \reprosim{9.5}{0.8} & \reprosim{8.7}{0.9} & \reprosim{0.2}{0.2} & \reprosim{95.5}{1.1} & \reprosim{$-$3.1}{0.2}\\
Soft-KD & \reprosim{83.5}{0.4} & \reprosim{89.4}{0.6} & \reprosim{24.9}{0.5} & \reprosim{23.1}{1.0} & \reprosim{71.6}{0.6} & \reprosim{25.3}{0.9} & \reprosim{37.9}{1.0} & \reprosim{\best{52.8}}{0.6} & \reprosim{75.1}{0.3} & \reprosim{12.6}{2.5} & \reprosim{10.9}{1.1} & \reprosim{\best{0.7}}{0.4} & \reprosim{\best{98.4}}{1.3} & \reprosim{$-$1.6}{0.5}\\
Off-MOPD & \reprosim{\best{84.0}}{0.2} & \reprosim{\best{89.7}}{0.2} & \reprosim{\best{25.0}}{2.7} & \reprosim{21.3}{1.4} & \reprosim{70.4}{0.8} & \reprosim{\best{26.7}}{0.9} & \reprosim{37.9}{0.5} & \reprosim{49.3}{0.5} & \reprosim{73.8}{1.7} & \reprosim{\best{20.0}}{0.7} & \reprosim{10.4}{1.3} & \reprosim{0.2}{0.3} & \reprosim{94.9}{4.4} & \reprosim{$-$1.7}{0.3}\\
MOPD & \reprosim{83.1}{0.3} & \reprosim{88.7}{0.3} & \reprosim{23.5}{1.3} & \reprosim{21.1}{1.2} & \reprosim{\best{73.0}}{0.8} & \reprosim{26.6}{1.3} & \reprosim{38.2}{2.5} & \reprosim{51.9}{0.6} & \reprosim{\best{76.1}}{1.2} & \reprosim{14.8}{1.5} & \reprosim{\best{13.2}}{1.1} & \reprosim{\best{0.7}}{0.2} & \reprosim{97.0}{2.6} & \reprosim{\best{$-$1.2}}{0.6}\\
\addlinespace[3pt]
\midrule
\tableheadrow\textbf{\boldmath\agentthree{}} & \recoveryOOD{MATH} & \recoveryOOD{GSM} & \recoveryOOD{AIME24} & \recoveryOOD{AIME25} & \recoveryOOD{IFEval} & \recoveryOOD{IFBen} & \recoveryOOD{GPQA} & \recoveryOOD{MMLU} & \recoveryID{IC-SQL} & \recoveryID{ALF-s} & \recoveryID{ALF-u} & \recoveryID{WShop} & \colhead{ID Rec.} & \colhead{OOD $\Delta$}\\
\midrule
Base & \reprosim{85.4}{0.3} & \reprosim{91.5}{0.0} & \reprosim{23.5}{0.3} & \reprosim{21.9}{0.4} & \reprosim{56.0}{0.2} & \reprosim{17.7}{0.3} & \reprosim{41.4}{0.5} & \reprosim{60.1}{0.0} & \reprosim{40.4}{0.1} & \reprosim{6.4}{0.7} & \reprosim{9.0}{0.7} & \reprosim{0.4}{0.2} & \reproref{\na} & \reproref{0.0}\\
Teachers & \reprosim{79.0}{0.3} & \reprosim{89.2}{0.0} & \reprosim{20.2}{0.8} & \reprosim{16.5}{0.4} & \reprosim{51.8}{0.2} & \reprosim{15.0}{0.3} & \reprosim{36.4}{0.5} & \reprosim{54.0}{0.0} & \reprosim{73.7}{0.1} & \reprosim{77.9}{0.7} & \reprosim{75.4}{0.7} & \reprosim{68.8}{0.2} & \reproref{100.0} & \reproref{\na}\\
\midrule[.3pt]
Mix-RL & \reprosim{79.9}{0.2} & \reprosim{91.4}{0.3} & \reprosim{19.5}{0.9} & \reprosim{14.9}{1.1} & \reprosim{53.4}{1.0} & \reprosim{16.2}{0.7} & \reprosim{33.2}{2.1} & \reprosim{54.6}{0.2} & \reprosim{66.3}{1.1} & \reprosim{77.1}{1.2} & \reprosim{74.6}{1.5} & \reprosim{64.6}{1.1} & \reprosim{92.7}{1.3} & \reprosim{$-$4.3}{0.5}\\
Merge & \reprosim{70.6}{0.5} & \reprosim{88.3}{0.1} & \reprosim{18.7}{0.4} & \reprosim{16.6}{0.7} & \reprosim{48.6}{0.2} & \reprosim{21.7}{0.3} & \reprosim{35.5}{0.6} & \reprosim{46.6}{0.0} & \reprosim{70.4}{0.1} & \reprosim{60.0}{0.7} & \reprosim{54.5}{0.7} & \reprosim{53.2}{0.5} & \reprosim{77.6}{0.4} & \reprosim{$-$6.4}{0.2}\\
\midrule[.3pt]
SFT & \reprosim{80.6}{0.5} & \reprosim{89.3}{0.3} & \reprosim{20.1}{0.8} & \reprosim{19.0}{1.9} & \reprosim{58.0}{0.3} & \reprosim{19.2}{0.4} & \reprosim{36.4}{0.5} & \reprosim{55.0}{0.2} & \reprosim{65.4}{0.3} & \reprosim{83.6}{0.7} & \reprosim{79.9}{0.7} & \reprosim{63.6}{0.2} & \reprosim{95.1}{0.3} & \reprosim{\best{$-$2.5}}{0.5}\\
Soft-KD & \reprosim{77.4}{1.4} & \reprosim{87.3}{0.8} & \reprosim{18.2}{2.1} & \reprosim{\best{19.5}}{1.3} & \reprosim{\best{59.8}}{1.0} & \reprosim{\best{22.3}}{1.2} & \reprosim{33.8}{1.8} & \reprosim{53.5}{0.3} & \reprosim{63.3}{0.1} & \reprosim{\best{84.3}}{1.9} & \reprosim{\best{80.6}}{2.0} & \reprosim{62.0}{0.8} & \reprosim{93.3}{1.2} & \reprosim{$-$3.2}{0.7}\\
Off-MOPD & \reprosim{79.9}{1.4} & \reprosim{\best{89.9}}{0.5} & \reprosim{\best{20.3}}{1.3} & \reprosim{13.2}{0.8} & \reprosim{54.0}{0.2} & \reprosim{14.3}{0.9} & \reprosim{32.7}{1.1} & \reprosim{54.9}{1.1} & \reprosim{\best{72.7}}{1.0} & \reprosim{77.1}{2.1} & \reprosim{68.7}{2.7} & \reprosim{65.6}{0.6} & \reprosim{95.4}{1.3} & \reprosim{$-$4.8}{0.7}\\
MOPD & \reprosim{\best{80.7}}{0.2} & \reprosim{89.7}{0.3} & \reprosim{20.1}{1.0} & \reprosim{16.5}{0.8} & \reprosim{53.9}{0.7} & \reprosim{16.8}{0.4} & \reprosim{\best{36.7}}{1.1} & \reprosim{\best{55.4}}{0.3} & \reprosim{68.8}{1.1} & \reprosim{81.4}{1.2} & \reprosim{77.6}{0.7} & \reprosim{\best{67.2}}{1.7} & \reprosim{\best{97.9}}{1.9} & \reprosim{$-$3.5}{0.1}\\
\bottomrule
\end{tabularx}
\endgroup
\end{table}

\begin{table}[H]
\papertablestyle
\caption{\textbf{Per-benchmark results for \mainmodel{}.}
Notation follows \cref{tab:capability-recovery}.
ID Rec.\ and OOD~$\Delta$ pool ALFWorld's seen (ALF-s) and unseen (ALF-u) splits. }
\label{tab:recovery-qwen35-4b-rerun}
\footnotesize
\centering
\begingroup
\setlength{\parskip}{0pt}
\fontsize{6.2}{7.4}\selectfont
\setlength{\tabcolsep}{0.8pt}
\begin{tabularx}{\linewidth}{l*{12}{>{\hsize=0.99\hsize\linewidth=\hsize}R}>{\hsize=1.12\hsize\linewidth=\hsize}R R}
\toprule
 & \multicolumn{8}{c}{\textbf{Single-turn}} & \multicolumn{4}{c}{\textbf{Multi-turn}} & \multicolumn{2}{c}{\textbf{Overall}}\\
\cmidrule(lr){2-9}\cmidrule(lr){10-13}\cmidrule(l){14-15}
\tableheadrow\textbf{\mia{}} & \recoveryID{MATH} & \recoveryID{GSM} & \recoveryID{AIME24} & \recoveryID{AIME25} & \recoveryID{IFEval} & \recoveryID{IFBen} & \recoveryOOD{GPQA} & \recoveryOOD{MMLU} & \recoveryID{IC-SQL} & \recoveryOOD{ALF-s} & \recoveryOOD{ALF-u} & \recoveryOOD{WShop} & \colhead{ID Rec.} & \colhead{OOD $\Delta$}\\
\midrule
Base & \reprosim{88.2}{0.2} & \reprosim{92.9}{0.1} & \reprosim{68.1}{0.8} & \reprosim{58.8}{1.5} & \reprosim{65.2}{0.1} & \reprosim{20.7}{0.3} & \reprosim{49.0}{0.5} & \reprosim{69.5}{0.0} & \reprosim{61.5}{0.1} & \reprosim{20.7}{0.7} & \reprosim{14.2}{0.7} & \reprosim{1.6}{0.2} & \reproref{\na} & \reproref{0.0}\\
Teachers & \reprosim{92.4}{0.2} & \reprosim{94.4}{0.1} & \reprosim{68.5}{0.9} & \reprosim{65.0}{1.4} & \reprosim{80.6}{0.2} & \reprosim{39.3}{0.3} & \reprosim{60.6}{0.5} & \reprosim{74.5}{0.0} & \reprosim{77.5}{0.1} & \reprosim{21.4}{0.7} & \reprosim{20.1}{0.7} & \reprosim{2.4}{0.2} & \reproref{100.0} & \reproref{\na}\\
\midrule[.3pt]
Mix-RL & \reprosim{92.6}{0.4} & \reprosim{97.2}{0.2} & \reprosim{68.9}{1.2} & \reprosim{64.2}{1.5} & \reprosim{80.5}{0.9} & \reprosim{39.9}{1.2} & \reprosim{55.1}{1.7} & \reprosim{70.8}{0.1} & \reprosim{76.6}{1.1} & \reprosim{17.6}{2.5} & \reprosim{20.6}{1.9} & \reprosim{0.3}{0.3} & \reprosim{103.5}{4.1} & \reprosim{$+$1.9}{0.6}\\
Merge & \reprosim{94.4}{0.2} & \reprosim{94.6}{0.1} & \reprosim{68.2}{1.3} & \reprosim{65.7}{1.5} & \reprosim{82.8}{0.2} & \reprosim{35.0}{0.3} & \reprosim{55.2}{0.6} & \reprosim{70.3}{0.0} & \reprosim{78.1}{0.1} & \reprosim{14.3}{0.7} & \reprosim{8.2}{0.7} & \reprosim{0.4}{0.2} & \reprosim{101.8}{2.5} & \reprosim{$-$0.1}{0.2}\\
\midrule[.3pt]
SFT & \reprosim{\best{93.9}}{0.3} & \reprosim{93.7}{0.2} & \reprosim{65.2}{2.2} & \reprosim{64.4}{1.3} & \reprosim{80.9}{0.6} & \reprosim{38.1}{1.0} & \reprosim{52.9}{1.1} & \reprosim{\best{67.6}}{0.2} & \reprosim{\best{80.6}}{0.4} & \reprosim{12.6}{1.5} & \reprosim{\best{14.7}}{1.1} & \reprosim{\best{0.7}}{0.3} & \reprosim{98.6}{3.0} & \reprosim{$-$0.7}{0.5}\\
Soft-KD & \reprosim{92.2}{0.5} & \reprosim{94.3}{0.1} & \reprosim{67.1}{1.1} & \reprosim{65.4}{1.7} & \reprosim{\best{81.3}}{0.4} & \reprosim{39.2}{1.3} & \reprosim{\best{53.4}}{1.3} & \reprosim{67.2}{0.2} & \reprosim{76.5}{0.1} & \reprosim{\best{16.0}}{1.8} & \reprosim{13.9}{1.6} & \reprosim{0.1}{0.2} & \reprosim{97.3}{3.4} & \reprosim{\best{$-$0.5}}{0.6}\\
Off-MOPD & \reprosim{92.5}{0.4} & \reprosim{94.2}{0.2} & \reprosim{66.7}{1.0} & \reprosim{64.7}{1.5} & \reprosim{81.0}{0.6} & \reprosim{\best{39.8}}{0.7} & \reprosim{52.0}{1.0} & \reprosim{65.0}{0.3} & \reprosim{78.0}{0.5} & \reprosim{5.5}{1.1} & \reprosim{12.7}{2.2} & \reprosim{0.1}{0.1} & \reprosim{\best{98.7}}{3.2} & \reprosim{$-$2.9}{0.6}\\
MOPD & \reprosim{93.1}{0.2} & \reprosim{\best{94.7}}{0.3} & \reprosim{\best{67.3}}{1.5} & \reprosim{\best{65.8}}{2.5} & \reprosim{80.3}{0.6} & \reprosim{38.8}{1.5} & \reprosim{51.3}{0.8} & \reprosim{67.0}{0.1} & \reprosim{76.3}{1.1} & \reprosim{9.5}{1.8} & \reprosim{10.2}{1.1} & \reprosim{0.3}{0.2} & \reprosim{97.8}{4.8} & \reprosim{$-$2.3}{0.4}\\
\addlinespace[3pt]
\midrule
\tableheadrow\textbf{\boldmath\agentthree{}} & \recoveryOOD{MATH} & \recoveryOOD{GSM} & \recoveryOOD{AIME24} & \recoveryOOD{AIME25} & \recoveryOOD{IFEval} & \recoveryOOD{IFBen} & \recoveryOOD{GPQA} & \recoveryOOD{MMLU} & \recoveryID{IC-SQL} & \recoveryID{ALF-s} & \recoveryID{ALF-u} & \recoveryID{WShop} & \colhead{ID Rec.} & \colhead{OOD $\Delta$}\\
\midrule
Base & \reprosim{88.2}{0.2} & \reprosim{92.9}{0.1} & \reprosim{68.1}{0.8} & \reprosim{58.8}{1.5} & \reprosim{65.2}{0.1} & \reprosim{20.7}{0.3} & \reprosim{49.0}{0.5} & \reprosim{69.5}{0.0} & \reprosim{61.5}{0.1} & \reprosim{20.7}{0.7} & \reprosim{14.2}{0.7} & \reprosim{1.6}{0.2} & \reproref{\na} & \reproref{0.0}\\
Teachers & \reprosim{88.2}{0.2} & \reprosim{93.6}{0.0} & \reprosim{69.4}{1.9} & \reprosim{60.2}{1.5} & \reprosim{67.3}{0.1} & \reprosim{21.3}{0.3} & \reprosim{56.6}{0.5} & \reprosim{70.8}{0.0} & \reprosim{77.5}{0.1} & \reprosim{85.7}{0.7} & \reprosim{79.1}{0.7} & \reprosim{79.6}{0.4} & \reproref{100.0} & \reproref{\na}\\
\midrule[.3pt]
Mix-RL & \reprosim{86.6}{0.3} & \reprosim{95.3}{0.3} & \reprosim{67.4}{1.0} & \reprosim{59.8}{1.6} & \reprosim{64.4}{1.3} & \reprosim{18.9}{1.1} & \reprosim{54.9}{0.8} & \reprosim{67.6}{0.2} & \reprosim{79.0}{0.5} & \reprosim{87.1}{1.2} & \reprosim{81.3}{2.2} & \reprosim{80.4}{1.1} & \reprosim{102.6}{1.2} & \reprosim{$+$0.3}{0.5}\\
Merge & \reprosim{83.5}{0.2} & \reprosim{91.1}{0.1} & \reprosim{46.5}{0.7} & \reprosim{37.8}{1.4} & \reprosim{63.4}{0.2} & \reprosim{18.1}{0.4} & \reprosim{52.0}{0.5} & \reprosim{69.4}{0.1} & \reprosim{73.5}{0.1} & \reprosim{77.1}{0.7} & \reprosim{70.1}{0.7} & \reprosim{71.4}{0.5} & \reprosim{86.8}{0.5} & \reprosim{$-$6.3}{0.2}\\
\midrule[.3pt]
SFT & \reprosim{81.9}{1.0} & \reprosim{\best{92.6}}{0.1} & \reprosim{48.7}{1.0} & \reprosim{38.8}{1.3} & \reprosim{\best{65.3}}{0.5} & \reprosim{18.6}{0.7} & \reprosim{\best{52.0}}{1.5} & \reprosim{67.1}{0.2} & \reprosim{76.9}{0.4} & \reprosim{84.3}{1.2} & \reprosim{\best{79.1}}{1.5} & \reprosim{75.2}{0.7} & \reprosim{96.4}{1.1} & \reprosim{$-$5.9}{0.5}\\
Soft-KD & \reprosim{78.3}{1.1} & \reprosim{91.6}{0.5} & \reprosim{41.8}{3.2} & \reprosim{37.5}{1.8} & \reprosim{60.5}{1.2} & \reprosim{16.8}{1.3} & \reprosim{43.8}{1.8} & \reprosim{65.2}{0.4} & \reprosim{\best{77.4}}{0.1} & \reprosim{\best{85.7}}{1.4} & \reprosim{74.6}{2.0} & \reprosim{77.2}{0.9} & \reprosim{\best{97.0}}{1.0} & \reprosim{$-$9.6}{1.0}\\
Off-MOPD & \reprosim{82.8}{0.9} & \reprosim{91.2}{0.4} & \reprosim{48.3}{2.5} & \reprosim{39.4}{2.9} & \reprosim{63.4}{0.7} & \reprosim{17.2}{0.5} & \reprosim{45.6}{1.8} & \reprosim{67.0}{0.4} & \reprosim{76.6}{1.3} & \reprosim{81.4}{1.4} & \reprosim{73.1}{1.3} & \reprosim{77.4}{0.6} & \reprosim{94.8}{1.3} & \reprosim{$-$7.2}{0.8}\\
MOPD & \reprosim{\best{82.9}}{0.4} & \reprosim{91.4}{0.1} & \reprosim{\best{53.3}}{2.1} & \reprosim{\best{43.8}}{2.1} & \reprosim{64.1}{0.3} & \reprosim{\best{20.0}}{0.6} & \reprosim{51.5}{1.0} & \reprosim{\best{68.0}}{0.1} & \reprosim{75.8}{1.4} & \reprosim{79.3}{0.7} & \reprosim{\best{79.1}}{2.0} & \reprosim{\best{78.0}}{1.2} & \reprosim{95.9}{1.6} & \reprosim{\best{$-$4.7}}{0.4}\\
\bottomrule
\end{tabularx}
\endgroup
\end{table}

\begin{table}[H]
\papertablestyle
\caption{\textbf{Per-benchmark results for \textsc{Qwen3-1.7B} with \textsc{Qwen3-8B} teachers.}
Notation follows \cref{tab:capability-recovery}.
ID Rec.\ and OOD~$\Delta$ pool ALFWorld's seen (ALF-s) and unseen (ALF-u) splits.
Base repeats \cref{tab:recovery-qwen3-1p7b} and Teachers repeats \cref{tab:recovery-qwen3-8b}, so ID Rec.\ is normalized by the 8B teachers' gains over the 1.7B base and is not directly comparable to \cref{tab:recovery-qwen3-1p7b}. }
\label{tab:recovery-qwen3-1p7b-8bteacher}
\footnotesize
\centering
\begingroup
\setlength{\parskip}{0pt}
\fontsize{6.2}{7.4}\selectfont
\setlength{\tabcolsep}{0.8pt}
\begin{tabularx}{\linewidth}{l*{12}{>{\hsize=0.98\hsize\linewidth=\hsize}R}>{\hsize=1.12\hsize\linewidth=\hsize}R@{\hspace{3.1pt}}>{\hsize=1.12\hsize\linewidth=\hsize}R}
\toprule
 & \multicolumn{8}{c}{\textbf{Single-turn}} & \multicolumn{4}{c}{\textbf{Multi-turn}} & \multicolumn{2}{c}{\textbf{Overall}}\\
\cmidrule(lr){2-9}\cmidrule(lr){10-13}\cmidrule(l){14-15}
\tableheadrow\textbf{\mia{}} & \recoveryID{MATH} & \recoveryID{GSM} & \recoveryID{AIME24} & \recoveryID{AIME25} & \recoveryID{IFEval} & \recoveryID{IFBen} & \recoveryOOD{GPQA} & \recoveryOOD{MMLU} & \recoveryID{IC-SQL} & \recoveryOOD{ALF-s} & \recoveryOOD{ALF-u} & \recoveryOOD{WShop} & \colhead{ID Rec.} & \colhead{OOD $\Delta$}\\
\midrule
Base & \reprosim{73.4}{0.9} & \reprosim{82.7}{0.1} & \reprosim{12.9}{0.4} & \reprosim{11.2}{0.8} & \reprosim{66.7}{0.1} & \reprosim{18.3}{0.3} & \reprosim{30.8}{0.5} & \reprosim{47.0}{0.1} & \reprosim{12.4}{0.1} & \reprosim{9.3}{0.7} & \reprosim{9.0}{0.7} & \reprosim{6.6}{0.2} & \reproref{\na} & \reproref{0.0}\\
Teachers & \reprosim{89.6}{0.7} & \reprosim{93.0}{0.1} & \reprosim{41.2}{1.5} & \reprosim{31.7}{0.8} & \reprosim{88.5}{0.2} & \reprosim{42.0}{0.3} & \reprosim{56.1}{0.5} & \reprosim{69.8}{0.1} & \reprosim{77.6}{0.1} & \reprosim{33.6}{1.2} & \reprosim{32.1}{0.7} & \reprosim{11.8}{0.2} & \reproref{100.0} & \reproref{\na}\\
\midrule[.3pt]
SFT & \reprosim{\best{88.3}}{0.3} & \reprosim{\best{91.2}}{0.3} & \reprosim{\best{38.0}}{1.7} & \reprosim{\best{29.9}}{2.4} & \reprosim{\best{88.2}}{0.4} & \reprosim{\best{40.1}}{0.8} & \reprosim{22.9}{1.5} & \reprosim{30.3}{0.8} & \reprosim{74.0}{0.5} & \reprosim{1.9}{0.8} & \reprosim{2.0}{1.1} & \reprosim{1.7}{0.2} & \reprosim{\best{92.5}}{1.6} & \reprosim{$-$9.2}{0.5}\\
Soft-KD & \reprosim{87.1}{0.3} & \reprosim{90.3}{0.4} & \reprosim{34.5}{0.7} & \reprosim{28.7}{1.7} & \reprosim{87.7}{0.5} & \reprosim{37.3}{1.0} & \reprosim{24.6}{0.8} & \reprosim{30.6}{1.2} & \reprosim{\best{76.8}}{0.1} & \reprosim{2.4}{0.8} & \reprosim{2.7}{0.9} & \reprosim{2.1}{0.5} & \reprosim{88.6}{1.5} & \reprosim{$-$8.4}{0.1}\\
Off-MOPD & \reprosim{86.0}{0.2} & \reprosim{90.3}{0.3} & \reprosim{33.9}{1.3} & \reprosim{27.4}{1.6} & \reprosim{85.4}{0.3} & \reprosim{32.4}{0.7} & \reprosim{\best{29.0}}{1.9} & \reprosim{42.9}{0.4} & \reprosim{72.6}{0.9} & \reprosim{\best{7.1}}{1.2} & \reprosim{\best{8.5}}{1.1} & \reprosim{1.9}{0.6} & \reprosim{80.9}{1.0} & \reprosim{$-$3.0}{0.8}\\
MOPD & \reprosim{85.3}{0.8} & \reprosim{89.0}{0.2} & \reprosim{31.7}{0.6} & \reprosim{24.7}{1.7} & \reprosim{85.5}{0.6} & \reprosim{33.1}{0.5} & \reprosim{28.6}{1.1} & \reprosim{\best{47.2}}{0.1} & \reprosim{72.3}{0.6} & \reprosim{4.3}{0.7} & \reprosim{6.7}{0.7} & \reprosim{\best{3.0}}{0.3} & \reprosim{77.4}{1.9} & \reprosim{\best{$-$2.3}}{0.3}\\
\addlinespace[3pt]
\midrule
\tableheadrow\textbf{\boldmath\agentthree{}} & \recoveryOOD{MATH} & \recoveryOOD{GSM} & \recoveryOOD{AIME24} & \recoveryOOD{AIME25} & \recoveryOOD{IFEval} & \recoveryOOD{IFBen} & \recoveryOOD{GPQA} & \recoveryOOD{MMLU} & \recoveryID{IC-SQL} & \recoveryID{ALF-s} & \recoveryID{ALF-u} & \recoveryID{WShop} & \colhead{ID Rec.} & \colhead{OOD $\Delta$}\\
\midrule
Base & \reprosim{73.4}{0.9} & \reprosim{82.7}{0.1} & \reprosim{12.9}{0.4} & \reprosim{11.2}{0.8} & \reprosim{66.7}{0.1} & \reprosim{18.3}{0.3} & \reprosim{30.8}{0.5} & \reprosim{47.0}{0.1} & \reprosim{12.4}{0.1} & \reprosim{9.3}{0.7} & \reprosim{9.0}{0.7} & \reprosim{6.6}{0.2} & \reproref{\na} & \reproref{0.0}\\
Teachers & \reprosim{84.8}{0.7} & \reprosim{92.7}{0.0} & \reprosim{29.4}{0.9} & \reprosim{22.3}{1.0} & \reprosim{84.5}{0.2} & \reprosim{27.7}{0.3} & \reprosim{49.0}{0.5} & \reprosim{65.3}{0.1} & \reprosim{78.6}{0.1} & \reprosim{82.1}{1.4} & \reprosim{80.6}{0.7} & \reprosim{62.2}{0.2} & \reproref{100.0} & \reproref{\na}\\
\midrule[.3pt]
SFT & \reprosim{60.0}{0.3} & \reprosim{78.0}{0.5} & \reprosim{1.3}{0.3} & \reprosim{1.1}{0.4} & \reprosim{42.5}{2.1} & \reprosim{11.6}{0.7} & \reprosim{9.9}{1.1} & \reprosim{20.2}{2.6} & \reprosim{\best{77.3}}{0.4} & \reprosim{\best{82.6}}{0.8} & \reprosim{\best{79.9}}{1.3} & \reprosim{\best{62.2}}{0.5} & \reprosim{\best{99.3}}{0.5} & \reprosim{$-$14.8}{0.7}\\
Soft-KD & \reprosim{67.0}{1.1} & \reprosim{75.1}{0.7} & \reprosim{0.5}{0.4} & \reprosim{0.8}{0.3} & \reprosim{41.8}{1.8} & \reprosim{17.4}{0.7} & \reprosim{9.8}{2.3} & \reprosim{37.0}{0.9} & \reprosim{70.9}{0.1} & \reprosim{77.4}{2.2} & \reprosim{74.6}{1.3} & \reprosim{61.6}{1.8} & \reprosim{93.0}{1.2} & \reprosim{$-$11.7}{0.7}\\
Off-MOPD & \reprosim{\best{72.6}}{1.0} & \reprosim{82.3}{0.4} & \reprosim{\best{13.5}}{1.0} & \reprosim{12.6}{1.0} & \reprosim{\best{62.8}}{1.4} & \reprosim{\best{17.6}}{0.8} & \reprosim{\best{25.6}}{2.0} & \reprosim{\best{46.9}}{0.1} & \reprosim{63.2}{1.6} & \reprosim{67.6}{0.8} & \reprosim{59.2}{1.1} & \reprosim{54.6}{1.1} & \reprosim{78.9}{0.9} & \reprosim{\best{$-$1.1}}{0.7}\\
MOPD & \reprosim{67.7}{1.4} & \reprosim{\best{82.9}}{0.5} & \reprosim{10.5}{1.0} & \reprosim{\best{12.7}}{1.1} & \reprosim{61.5}{1.5} & \reprosim{11.4}{0.4} & \reprosim{\best{25.6}}{1.3} & \reprosim{44.4}{0.1} & \reprosim{60.8}{0.7} & \reprosim{61.7}{2.1} & \reprosim{55.2}{0.7} & \reprosim{51.0}{0.7} & \reprosim{73.3}{0.3} & \reprosim{$-$3.3}{0.6}\\
\bottomrule
\end{tabularx}
\endgroup
\end{table}

\section{Efficiency and Compute}
\label{app:compute}

\subsection{Training-Step Breakdown}
\label{app:step-timing}
\Cref{tab:opd-step-timing} decomposes training time within each \textsc{Qwen3-8B} configuration. Seconds per step are not directly comparable across methods: Offline-MOPD and MOPD use batches of 24 and 96 prompts on \mia{}, respectively. The breakdown identifies the dominant component: student rollout accounts for 86.5--96.0\% of step time, while teacher scoring accounts for at most 2.6\%.

\begin{table}[H]
\papertablestyle
\caption{\textbf{OPD training-step timing for \textsc{Qwen3-8B}.}
Batch sizes (Offline-MOPD/MOPD) are 24/96 on \mia{} and 40/40 on \agentthree{}; s/step describes each configuration. Rollout includes student generation and live environment interaction where applicable. Teacher scoring comprises forward passes over student trajectories. }
\label{tab:opd-step-timing}
\footnotesize
\centering
\setlength{\tabcolsep}{3pt}
\begin{tabularx}{\linewidth}{@{}Yrrrr@{}}
\toprule
 & \multicolumn{2}{c}{\tableheading{Offline-MOPD}} & \multicolumn{2}{c}{\tableheading{MOPD}}\\
\cmidrule(lr){2-3}\cmidrule(lr){4-5}
\tableheadrow\tableheading{\mia{}} & \tableheading{s/step} & \tableheading{\% of step} & \tableheading{s/step} & \tableheading{\% of step}\\
\midrule
Student rollout & 60.89 & 95.6 & 276.23 & 86.5\\
Teacher scoring & 0.60 & 0.9 & 8.20 & 2.6\\
Student log-probabilities & 0.40 & 0.6 & 6.60 & 2.1\\
Student gradient update & 1.80 & 2.8 & 28.40 & 8.9\\
\midrule
Full step & 63.69 & 100.0 & 319.43 & 100.0\\
\midrule
\tableheadrow\tableheading{\agentthree{}} & \tableheading{s/step} & \tableheading{\% of step} & \tableheading{s/step} & \tableheading{\% of step}\\
\midrule
Student rollout & 640.61 & 96.0 & 720.83 & 92.5\\
Teacher scoring & 4.70 & 0.7 & 10.20 & 1.3\\
Student log-probabilities & 4.20 & 0.6 & 9.30 & 1.2\\
Student gradient update & 17.60 & 2.6 & 39.30 & 5.0\\
\midrule
Full step & 667.11 & 100.0 & 779.63 & 100.0\\
\bottomrule
\end{tabularx}
\end{table}

\subsection{Teacher Trajectory Generation}
\label{app:teacher-generation}
Teacher-corpus construction requires autoregressive decoding and, for agentic tasks, environment interaction (\cref{app:corpus}). SFT, Soft-KD, and Offline-MOPD reuse the corpus across hyperparameter settings and seeds. A generation cost of $C_{\mathrm{gen}}$ shared across $N$ runs contributes $C_{\mathrm{gen}}/N$ GPU-hours per run and is excluded from training costs. The following B200 costs equal eight times the recorded generation wall time in hours.
\begin{center}
\papertablestyle
\footnotesize
\setlength{\tabcolsep}{3pt}
\begin{tabularx}{\linewidth}{@{}Yrr@{}}
\toprule
\tableheadrow\tableheading{Teacher family} & \tableheading{\mia{} (B200 GPU-h)} & \tableheading{\agentthree{} (B200 GPU-h)}\\
\midrule
\textsc{Qwen3-1.7B} & 7.93 & 8.20\\
\textsc{Ministral-3-3B} & 3.41 & 22.39\\
\textsc{Qwen3.5-4B} & 14.49 & 66.37\\
\textsc{Qwen3-8B} & 13.55 & 45.27\\
\bottomrule
\end{tabularx}
\end{center}

MOPD and Offline-MOPD also require teacher forward passes to score student-generated trajectories during training. This scoring cost is included in training GPU-hours. Because corpus generation is excluded, training-cost ratios do not represent end-to-end cost ratios.

\subsection{Training GPU-Hours}
Each run uses eight B200 GPUs. Training GPU-hours exclude teacher-corpus generation, validation, test evaluation, and other hyperparameter-search runs. \Cref{tab:compute} reports cumulative costs at the full budget and the selected validation checkpoint. MOPD has the highest full-run cost in both setups: 383.32 GPU-hours on \mia{} and 415.80 on \agentthree{}, or $14.8\times$ and $23.1\times$ SFT's costs, respectively.
\begin{table}[H]
\centering
\papertablestyle

\footnotesize
\setlength{\tabcolsep}{3pt}
\caption{\textbf{Training configurations and checkpoint costs for \textsc{Qwen3-8B}.}
Budget is the number of updates in a completed run. OPD cost at update $k$ is $C(k)=8kT/3600$ GPU-hours, with constant step time $T$ from \cref{tab:opd-step-timing}. }\label{tab:compute}
\begin{tabularx}{\linewidth}{@{}Yrrrrrrr@{}}
\toprule
 & & & \multicolumn{2}{c}{\tableheading{Full run}} & \multicolumn{3}{c}{\tableheading{Best validation checkpoint}}\\
\cmidrule(lr){4-5}\cmidrule(lr){6-8}
\tableheadrow\tableheading{Method} & \tableheading{LR ($10^{-6}$)} & \tableheading{Budget} & \tableheading{Step} & \tableheading{GPU-h} & \tableheading{Step} & \tableheading{Dev (\%)} & \tableheading{GPU-h}\\

\midrule
\tableheadrow\multicolumn{8}{l}{\tableheading{\mia{}}}\\
SFT & 5 & 1920 & 1920 & 25.90 & 1056 & 65.15 & 14.57\\
Soft-KD & 6.3 & 1920 & 1920 & 51.90 & 1440 & 65.31 & 38.60\\
Off-MOPD & 1 & 1140 & 1140 & 161.35 & 1026 & 65.07 & 145.21\\
MOPD & 3 & 540 & 540 & 383.32 & 390 & 65.21 & 276.84\\
\midrule
\tableheadrow\multicolumn{8}{l}{\tableheading{\agentthree{}}}\\
\midrule
SFT & 5 & 1140 & 1140 & 18.00 & 285 & 80.26 & 4.50\\
Soft-KD & 5 & 1140 & 1140 & 25.60 & 912 & 81.23 & 20.44\\
Off-MOPD & 3 & 240 & 240 & 355.79 & 240 & 80.20 & 355.79\\
MOPD & 3 & 240 & 240 & 415.80 & 156 & 80.67 & 270.27\\
\bottomrule
\end{tabularx}
\end{table}

\section{Limitations}
\label{app:limitations}
Due to limited resources, our main study covers four models from 1.7B to 8B parameters and two multi-teacher settings. More diverse teacher setup beyond \mia{} and \agentthree{ } are not tested. The larger-teacher comparison uses one 1.7B student and 8B teachers from the same model family. These results do not establish whether off-policy distillation remains competitive at larger scales or with less compatible teachers.


\FloatBarrier

\appletrademark

\end{document}